\documentclass[runningheads]{llncs}

\usepackage[final,year=2026,ID=10081]{eccv}
\usepackage{eccv}

\usepackage{eccvabbrv}

\usepackage{graphicx}
\usepackage{booktabs}
\usepackage[table]{xcolor}
\usepackage[accsupp]{axessibility}  

\newtheorem{assumption}{Assumption}
\usepackage{algorithm}
\usepackage{algpseudocode}
\usepackage{multirow}
\usepackage{booktabs}
\usepackage{graphicx}

\usepackage[pagebackref,breaklinks,colorlinks,citecolor=eccvblue]{hyperref}
\usepackage{hyperref}

\usepackage{orcidlink}

\newcommand{\yr}{\textcolor{black}}

\begin{document}

\title{Cycle-World: Mitigating Error Accumulation in Long-term Video World Models via Reverse-Prediction Cycle Consistency} 

\titlerunning{Cycle-World}

\author{Zihan Su\inst{1}\thanks{Equal contribution, $^\dagger$ Corresponding authors.}\orcidlink{0009-0008-4612-2368}  \and
Teng Hu\inst{1}$^\star$\orcidlink{0009-0008-1247-5931}  \and
Jiangning Zhang\inst{2}\orcidlink{0000-0001-8891-6766} \and
Ruiyan Wang\inst{1}\orcidlink{0009-0007-8202-213X} \and
Ran Yi\inst{1}$^\dagger$\orcidlink{0000-0003-1858-3358} \and
Lizhuang Ma\inst{1}$^\dagger$\orcidlink{0000-0003-1653-4341} \and
Dacheng Tao\inst{3}\orcidlink{0000-0001-7225-5449} }

\authorrunning{Z.~Su et al.}

\institute{School of Computer Science, Shanghai Jiao Tong University, Shanghai, China \and
Institute of Cyber-Systems and Control,
Zhejiang University, Hangzhou, China \and Nanyang Technological University, Singapore
\url{https://szhcz.github.io/projects/Cycle-World/}
}

\maketitle

\begin{abstract}
Autoregressive diffusion models have enabled high-quality video generation, yet their sequential nature inherently suffers from error accumulation. In long-horizon \yr{video} synthesis, minor prediction deviations compound over time, inevitably leading to unconstrained generative drift, structural collapse, and severe visual degradation. To address this, we propose Cycle-World, a novel framework designed for stable and temporally consistent long-video generation. Our approach tackles error drift by enforcing strict temporal reversibility across both the training and inference phases. Theoretically, we demonstrate that forward generative drift can be strictly bottlenecked by a cycle-consistency objective. During training, we integrate an efficient reverse-prediction model to implicitly embed causal constraints into the forward generator, compelling it to produce reversible sequences that tightly adhere to the natural video manifold. At inference time, we repurpose this frozen reverse model as a runtime corrector. Through gradient-based cycle guidance, it iteratively refines the generated latent representations, actively \yr{suppressing} accumulated errors before they are committed to the historical context. Extensive experiments on the VBench benchmark demonstrate that Cycle-World's dual-phase synergy significantly mitigates error drift, achieving state-of-the-art overall generation quality and long-horizon temporal consistency in 60-second synthesis.
\end{abstract}    

\begin{figure}[h!]
    \centering
    \includegraphics[width=\textwidth]{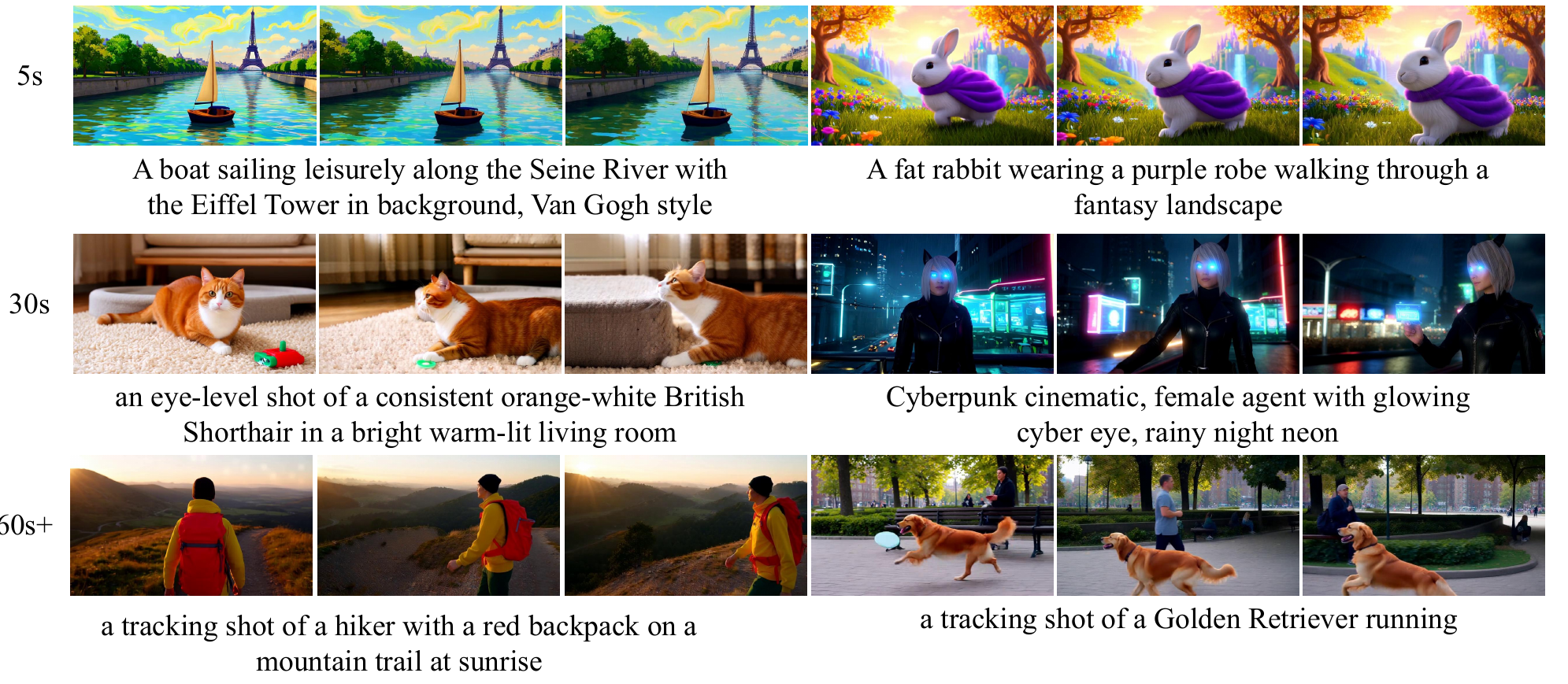}
\caption{\textbf{High-fidelity, long-horizon video generation with Cycle-World.} By effectively bottlenecking generative drift via temporal reversibility, our framework strictly suppresses structural hallucinations inherent in forward-only models. Cycle-World maintains state-of-the-art visual quality, strict physical conservation, and temporally consistent object states over extended generation horizons.}
    \label{fig:teaser}
\end{figure}

\section{Introduction}
\label{sec:intro}

The field of video generation has witnessed unprecedented advancements recently, driven by powerful models such as Sora\cite{sora-openai2024sora,sora2openai2025sora2}, Seedance\cite{seedance-gao2025seedance,seedance2}, and Kling\cite{kling}, alongside pioneering open-source efforts like Wan\cite{wan-2025wan}, HunyuanVideo\cite{hunyuanvideo-kong2024hunyuanvideo}.
While these models primarily rely on bidirectional or full-sequence architectures that achieve remarkable visual quality, their non-causal nature fundamentally limits their flexibility for open-ended, sequential, and interactive generation. As the community pivots towards the paradigm of World Models\cite{hyworld2025,2026matrix,mao2025yume,MetaWorld}, there is a critical consensus that real-time interactivity, continuous generation, and causal reasoning are indispensable. Consequently, the field is experiencing a paradigm shift towards causal, autoregressive generation models\cite{causvid-yin2025slow,selfforcing-huang2025self,zhu2026causalforcing,a-causal-cnn}.

However, this shift towards causal autoregressive models introduces severe bottlenecks in \yr{long video generation}. A primary challenge is the train-inference mismatch, often leading to rapid accumulation \yr{of errors}. Traditionally, models are trained using \textit{Teacher-Forcing}, which strictly relies on ground-truth past frames, causing significant exposure bias: during inference, \yr{the autoregressive model relies on its own} past predictions, \yr{where} minor deviations \yr{--- unseen during training --- propagate and amplify}. To bridge this gap and mitigate drift, the community has continuously evolved towards more robust training paradigms. This progression spans from the introduction of \textit{Diffusion Forcing}\cite{Diffusion-forcing-chen2024diffusion} to recent advanced train-inference alignment strategies, such as \textit{Self-Forcing}\cite{selfforcing-huang2025self} and \textit{LongLive}\cite{longlive-yang2025longlive}. 

\yr{Nevertheless}, we observe that even with these sophisticated mitigations, forward-only generation still suffers from a more fundamental and destructive issue: \textit{structural hallucinations}. While \yr{existing} progressive methods effectively suppress generic noise accumulation, they operate strictly in a unidirectional, forward-time manner. Crucially, they lack \textbf{\yr{temporal} cycle consistency}---the explicit temporal constraints required to ensure that a generated state logically allows its past to have happened\yr{, such that past states could be reversely predicted from future states}. Without this bidirectional verification, the models lack the physical constraints necessary to maintain continuous object states. Consequently, they are prone to severe video distortion and non-physical artifacts---such as characters clipping through solid objects, entities spontaneously appearing, or objects vanishing without a trace. Unlike generic noise, these structural hallucinations represent irreversible violations of real-world physics that abruptly destroy the integrity of the generated sequence.

To fundamentally address these irreversible structural hallucinations, we propose \textbf{Cycle-World}, a unified framework that enforces physical consistency via temporal reversibility. Our core motivation is grounded in a \yr{fundamental premise}: \textit{if a generated causal sequence obeys real-world dynamics, it must be temporally reversible}. Based on this insight, we first establish the \textbf{Cycle-Bounded Drift (CBD)} theorem, \yr{formally} proving that the unconstrained error accumulation in forward autoregressive synthesis can be strictly bottlenecked by minimizing the reverse reconstruction error. 
Guided by this theoretical guarantee, we translate the mathematical bound into a practical \textbf{Cycle-Consistent Learning (CCL)} paradigm. By introducing a reverse-prediction branch, we explicitly constrain the forward generator to produce inherently reversible latents, enabling it to foresee and \yr{suppress} non-physical artifacts during training. Furthermore, to extend the applicability of our theory to pre-trained models and resource-constrained scenarios, we propose \textbf{Cycle-Guided Inference (CGI)}. This \yr{inference-time} strategy repurposes the reverse model as a runtime critic to iteratively refine latents, offering a plug-and-play solution that significantly boosts stability without the need for expensive architectural modifications.

Extensive experiments conducted on VBench validate the effectiveness of our framework. Cycle-World significantly mitigates error drift, achieving state-of-the-art visual fidelity, semantic consistency, and overall temporal stability in long-video generation, as shown in Fig.~\ref{fig:teaser}.

In summary, our main contributions are threefold:
\begin{itemize}
    \item We establish the \textbf{Cycle-Bounded Drift (CBD)} theorem, a theoretical framework demonstrating that unconstrained generative drift and structural hallucinations in forward autoregressive synthesis can be strictly bottlenecked by enforcing temporal reversibility.
    \item We propose a \textbf{Cycle-Consistent Learning (CCL)} paradigm. By introducing a novel reverse-prediction cycle-consistency loss ($\mathcal{L}_{cycle}$), we explicitly constrain the forward causal generator to \yr{internalize} physical conservation and maintain long-term structural integrity.
    \item We introduce \textbf{Cycle-Guided Inference (CGI)}, a zero-shot runtime optimization strategy. By repurposing the frozen reverse model as a runtime critic, CGI utilizes iterative gradient-based latent refinement to actively \yr{rectify} non-physical artifacts, significantly enhancing long-\yr{term generation} quality without architectural modifications to the forward model.
\end{itemize}

\section{Related Work}
\label{sec:related_work}
\subsection{Video Generation} 
Diffusion models\cite{ddpm-ho2020denoising,ddim-song2020denoising,iddpm-nichol2021improved}, particularly Diffusion Transformers (DiT)\cite{dit-peebles2023scalable}, have established the prevailing paradigm for video generation\cite{wan-2025wan,hunyuanvideo-kong2024hunyuanvideo,cogvideox-yang2024cogvideox,Open-sora-zheng2024open,hu2026evolutionvideogenerativefoundations}, achieving remarkable visual fidelity. This rapid progress spans various specialized capabilities, including multimodal customized generation\cite{hu2025hunyuancustommultimodaldrivenarchitecturecustomized,PolyVivid,OmniVCus}, and native high-resolution synthesis\cite{wang2025lingen,hu2025ultragenhighresolutionvideogeneration}, and spatiotemporally consistent video processing\cite{Temporally-consistent-video-colorization,Full-Frame-Video-Stabilization}. Furthermore, recent advancements have expanded into joint audio-visual generation, leveraging cross-modal interactions and cross-task synergy to achieve synchronized multi-sensory synthesis\cite{zhang2026uniavgen,Harmony,low2025Ovi,wang2025universe,liu2026javisdit++}. However, despite these visual and multimodal achievements, their reliance on non-causal, bidirectional attention for concurrent frame denoising restricts their flexibility for open-ended generation. Conversely, pure autoregressive (AR) models\cite{videopoet-kondratyuk2023videopoet,cogvideo-hong2022cogvideo} offer sequential flexibility via discrete next-token prediction but suffer from irreversible information loss during latent compression and suboptimal generation diversity. 

To bridge this gap, recent works explore hybrid Autoregressive Diffusion architectures\cite{magi-1-dobrosotskaya2000magi,Pyramid-Flow-jin2024pyramidal,nova-deng2024autoregressive,causvid-yin2025slow,selfforcing-huang2025self}, which temporally decompose generation by conditioning future frames on past context. Despite their promise, these models suffer from severe exposure bias. During iterative inference, conditioning on imperfect prior predictions causes minor deviations to amplify continuously. This catastrophic error accumulation constitutes the primary bottleneck our work addresses.

\subsection{Long Video Generation}
Early long-video approaches relied on generating overlapping clips\cite{Seine-chen2023seine,freenoise-qiu2023freenoise,Gen-l-video-wang2023gen} or performing temporal interpolation between sparse keyframes\cite{Nuwa-xl-yin2023nuwa,Align-your-latents-blattmann2023align}. While extending the generation window, these heuristic designs fail to achieve true, infinitely long streaming synthesis. Consequently, the focus has shifted toward native causal modeling. However, traditional training paradigms for these models often employ Teacher Forcing\cite{TimeGrad-rasul2021autoregressive,teacher-forcing-williams1989learning}, which inevitably introduces a severe distribution discrepancy between the training and inference stages.
\begin{figure}[t]
    \centering
    \includegraphics[width=\textwidth]{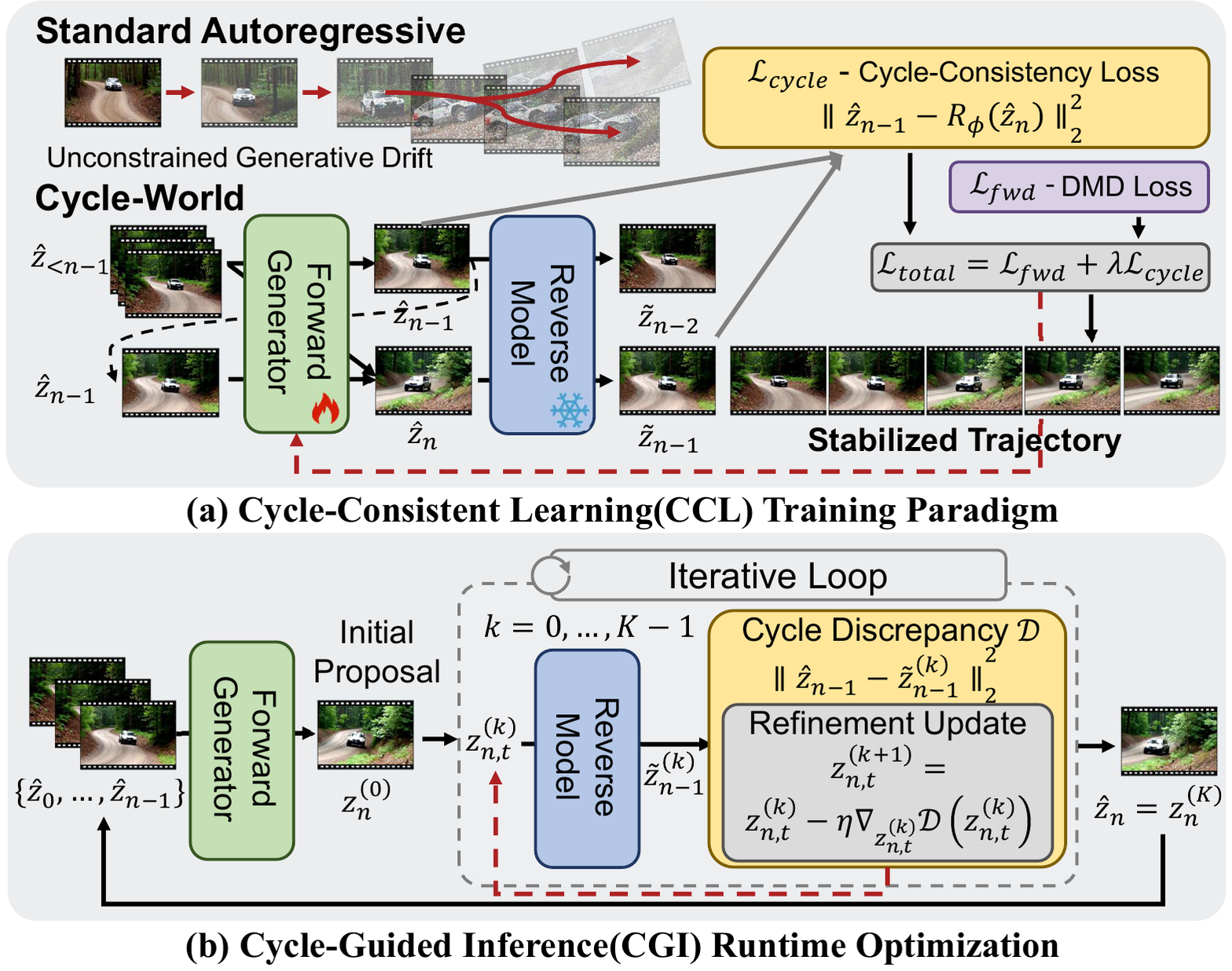}
\caption{\textbf{Overview of Cycle-World.} We mitigate generative drift in autoregressive video synthesis by enforcing temporal reversibility. \textbf{(a) Training (CCL):} The forward generator $G_\theta$ and reverse model $R_\phi$ are jointly optimized. A latent cycle-consistency loss ($\mathcal{L}_{cycle}$) explicitly penalizes physically irreversible trajectories. \textbf{(b) Inference (CGI):} The frozen $R_\phi$ acts as a runtime corrector. By evaluating cycle discrepancy ($\mathcal{D}$), it performs iterative gradient-based latent refinement to actively prune accumulated errors before they enter the historical context.}
    \label{fig:framework}
\end{figure}
To alleviate this issue, Diffusion Forcing\cite{Diffusion-forcing-chen2024diffusion} proposes joint denoising optimization for tokens with independent noise levels, which SkyReels-V2\cite{Skyreels-v2-chen2025skyreels} further integrates with Multi-modal Large Language Models to facilitate infinite-length, cinematic video synthesis. CausVid\cite{causvid-yin2025slow} extends distribution matching distillation\cite{dmd-yin2024one,dmd2-yin2024improved} to the video domain to mitigate error accumulation in AR generation. To resolve exposure bias and bridge the train-test gap, Self-Forcing\cite{selfforcing-huang2025self} innovatively simulates inference conditions directly during training by executing autoregressive rollouts with a Key-Value cache,  optimizing the model conditioned on its own historically generated outputs. Most recent cutting-edge works build upon this foundation. LongLive\cite{longlive-yang2025longlive}, for instance, adopts streaming long-sequence fine-tuning to strictly maintain end-to-end "train-long-test-long" consistency.

\section{Methodology}
\label{sec:method}
The overarching goal of \textbf{Cycle-World} is to mitigate the unconstrained generative drift and structural hallucinations inherent in \yr{long-term} autoregressive video synthesis. To achieve this, we introduce a novel framework grounded in the principle of temporal reversibility---the intuition that physically valid and structurally sound video dynamics must be accurately reversible. As illustrated in Figure \ref{fig:framework}, our approach \yr{addresses} this challenge through a cohesive pipeline encompassing theoretical grounding, cycle-consistent learning, and inference-time optimization. Specifically, we first establish the \textbf{Cycle-Bounded Drift (CBD)} theorem (Sec.~\ref{sec:theory}), a theoretical foundation demonstrating that the forward generative error can be strictly constrained by minimizing a reverse-prediction cycle-consistency error. Guided by this insight, we propose a \textbf{Cycle-Consistent Learning (CCL)} paradigm (Sec.~\ref{sec:training-method}). In this phase, we introduce a reverse-prediction model $R_\phi$ to enforce a cycle-consistency loss ($\mathcal{L}_{cycle}$) on the forward causal generator $G_\theta$, explicitly penalizing irreversible structural hallucinations. Finally, to combat out-of-distribution perturbations during open-ended generation, we introduce \textbf{Cycle-Guided Inference} (Sec.~\ref{sec:inference-method}). Under this strategy, we repurpose the frozen reverse model as a runtime corrector. By actively evaluating cycle discrepancy, it performs iterative gradient-based latent refinement to prune non-physical artifacts before they are committed to the historical context.

\subsection{Theoretical Foundation: Bounding Generative Drift via Temporal Reversibility}
\label{sec:theory}
Let $Z = \{z_1, z_2, \dots, z_N\}$ denote the ground-truth latent sequence of a natural video, where each $z_n$ is a latent chunk. In standard autoregressive synthesis, a forward causal generator $G_\theta$ sequentially predicts $\hat{z}_n$ conditioned on the previously generated context $\hat{z}_{<n}$. Because this sequential generation is inherently unconstrained, minor prediction deviations compound at each step $n$, leading to unbounded generative drift and  structural collapse in \yr{long video} synthesis.

We theoretically argue that this generative drift can be strictly constrained by enforcing temporal reversibility. Because natural video dynamics obey physical laws and spatiotemporal causality, they are inherently reversible. If a generated frame $\hat{z}_n$ severely deviates from the natural manifold, it loses the causal information necessary to reconstruct its history.

Let $e_n = \left\| \hat{z}_n - z_n \right\|$ be the accumulated generative drift at step $n$. To constrain this, we introduce a reverse-prediction model $R_\phi$ mapping a state at $n$ back to $n-1$. We establish our theoretical guarantees based on two mild assumptions:

\begin{assumption}[Reverse Predictability of Natural Dynamics]
\label{assump:reverse_predictability}
The reverse model $R_\phi$ is comprehensively trained under a self-forcing paradigm\cite{selfforcing-huang2025self}. Given this aligned training scheme, its approximation error for short-horizon reverse prediction is bounded by a small constant $\epsilon_R > 0$. That is, $\left\| R_\phi(z_n) - z_{n-1} \right\| \le \epsilon_R$.
\end{assumption}

\begin{assumption}[Reverse-Lipschitz Continuity]
\label{assump:reverse_lipschitz}
The learned reverse mapping $R_\phi$ preserves the distance properties within the relevant latent manifold, satisfying a reverse-Lipschitz condition. There exists a constant $C > 0$ such that for any two states $z_a$ and $z_b$, $\left\| z_a - z_b \right\| \le C \left\| R_\phi(z_a) - R_\phi(z_b) \right\|$.
\end{assumption}

Under these conditions, we demonstrate that the forward generative error $e_n$ is constrained by the cycle-consistency objective and the error from the previous steps. We formally define the single-step cycle-consistency distance as $d_{cycle}^{(n)} = \left\| \hat{z}_{n-1} - R_\phi(\hat{z}_n) \right\|$. 

Due to space constraints, all detailed proofs for the theoretical results presented in this section are deferred to Supplementary Material.
\begin{theorem}[Cycle-Bounded Drift]
\label{thm:cycle_error_bound}
Under Assumptions \ref{assump:reverse_predictability} and \ref{assump:reverse_lipschitz}, the forward generative drift $e_n$ satisfies:
\begin{equation}
    e_n \le C \left(d_{cycle}^{(n)} + e_{n-1} + \epsilon_R \right).
\end{equation}
\end{theorem}

To understand the compounding drift over a long horizon, we recursively unroll this step-wise recurrence.

\begin{corollary}[Long-Horizon Error Bound]
\label{cor:cumulative_drift_bound}
Assuming a worst-case upper limit on the single-step cycle-consistency distance, $\max_{i} d_{cycle}^{(i)} \le \delta_{cycle}$, the total generative drift at step $n$ (for $C \neq 1$) is explicitly governed by:

\begin{equation}
    \label{eq:closed_form_bound}
    e_n \le C^n e_0 + (\delta_{cycle} + \epsilon_R) C \frac{C^n - 1}{C - 1}.
\end{equation}
\end{corollary}

Building on Corollary \ref{cor:cumulative_drift_bound}, \yr{next we show} Cycle-World framework \yr{is} theoretical\yr{ly} superior \yr{to} unconstrained baselines. Let $\delta_{unc} = \max_i \| \hat{z}^{unc}_{i-1} - R_\phi(\hat{z}^{unc}_i) \|$ represent the maximum local cycle error of a standard, unconstrained generator, and $\delta_{cycle}$ represent our constrained error, where $\delta_{cycle} \ll \delta_{unc}$.

\begin{proposition}[Theoretical Advantage over Unconstrained Baselines]
\label{prop:error_gap}
Given the identical initial drift condition $e_0$ and the same pre-trained reverse model $R_\phi$ (with properties defined in Assumptions \ref{assump:reverse_predictability} and \ref{assump:reverse_lipschitz}), let $E_n^{unc}$ and $E_n^{ours}$ denote the theoretical upper limits of the generative drift at step $n$ for the unconstrained baseline and our constrained method, respectively. The explicit drift reduction (the gap between these theoretical limits) achieved by our method is:
\begin{equation}
    \Delta E_n = E_n^{unc} - E_n^{ours} = (\delta_{unc} - \delta_{cycle}) C \frac{C^n - 1}{C - 1} >0.
\end{equation}
\end{proposition}

\textbf{Remark.} Proposition \ref{prop:error_gap} is the theoretical cornerstone of our method. The term $(\delta_{unc} - \delta_{cycle})$ represents the single-step advantage gained by explicitly enforcing temporal reversibility. Crucially, the multiplier $C \frac{C^n - 1}{C - 1}$ implies that this advantage is not merely additive, but scales with the sequence length $n$. This proves that while standard generation and our method might perform similarly for very short clips, the unconstrained baseline will inevitably suffer from a much faster error explosion over time. By minimizing the single-step cycle error upper bound $\delta_{cycle}$, Cycle-World effectively suppresses the base magnitude of the accumulated generative drift, theoretically guaranteeing significantly better structural preservation in long-video generation.

\subsection{Cycle-Consistent Learning: Enforcing Cycle Consistency via Reverse Prediction}

\label{sec:training-method}

\textbf{Constraining the Forward Model \yr{via Temporal Cycle Consistency}.}
Building upon the theoretical guarantees established in Section \ref{sec:theory}—specifically Proposition \ref{prop:error_gap}, which proves that bounding single-step cycle error mathematically curtails compounding generative drift—we translate this insight into a practical learning framework. we extend a forward autoregressive video generator $G_\theta$ predicting latent chunks $\hat{z}_n$ given history $\hat{z}_{<n}$. Since standard maximum likelihood training optimizes solely forward prediction, unconstrained reverse consistency causes unbounded error accumulation scaling with $C \frac{C^n - 1}{C - 1}$ over sequence length $n$ (Corollary \ref{cor:cumulative_drift_bound}). To minimize this bound and mitigate structural hallucinations, we introduce a reverse-prediction model $R_\phi$ enforcing temporal cycle consistency.

\noindent\textbf{The Pixel-Latent Mismatch Bottleneck.}
Implementing $R_\phi$ naively by training a separate autoregressive model on reversed videos incurs severe dual distillation overhead and faces a prohibitive structural barrier: pixel-latent mismatch. Because Video VAEs temporally compress continuous frame sequences into a \yr{compact} latent \yr{representation}, the latent \yr{features} of a forward sequence 
\yr{do not exhibit simple temporal symmetry with those}
of a reversed sequence. This mismatch precludes the direct computation of the cycle-consistency distance, $d_{cycle}^{(n)} = \|\hat{z}_{n-1} - R_\phi(\hat{z}_n)\|$ defined in Theorem \ref{thm:cycle_error_bound}. Aligning the forward-ordered chunk $\hat{z}_{n-1}$ and a reverse-predicted chunk residing in disjoint manifolds requires an exorbitant Decode-Flip-Encode loop of decoding, temporally flipping, and re-encoding, rendering joint training computationally infeasible.

\noindent\textbf{Our Solution: Intrinsic Latent Reversibility.}
To fundamentally circumvent this bottleneck and the inherent feature mismatch, we propose \textit{Intrinsic Latent Reversibility}, redefining the cycle consistency objective directly on the intrinsic latent manifold rather than mapping back to the extrinsic pixel domain. $R_\phi$ \yr{need} not predict decoded, temporally flipped frames. \yr{Instead, we formulate the reverse task as learning the inverse transition dynamics within the latent space.} We set the \yr{optimization} target of $R_\phi(\hat{z}_n)$ to the forward-ordered history $\hat{z}_{n-1}$, learning the inverse probability $P(z_{n-1} | \hat{z}_n)$ directly. This approach \yr{eliminates the need for ``Decode-Flip-Encode" loop. Consequently,} the cycle-consistency distance ${d}_{cycle}^{(n)}$ \yr{reduces to a direct} vector subtraction in the shared latent space. \yr{This allows us to enforce strict structural constraints with negligible computational overhead, making joint training feasible.}

\noindent\textbf{Forward Generation via Self-Forcing.} 
To bridge the train-inference discrepancy, we train $G_\theta$ via self-forcing to predict the current chunk $\hat{z}_n$ given its generated history $\hat{z}_{<n}$ instead of ground truth. Since Distribution Matching Distillation (DMD) lacks paired regression constraints, optimizing this process updates the forward objective $\mathcal{L}_{fwd}$ using its score-difference gradient:
\begin{equation}
    \nabla_\theta \mathcal{L}_{fwd} = -\mathbb{E}_{t, \epsilon} \left[ \left( s_{real}(\hat{z}_{n,t}, t) - s_{fake}(\hat{z}_{n,t}, t) \right) \frac{\partial G_\theta(\hat{z}_{<n})}{\partial \theta} \right],
\end{equation}
where $\hat{z}_{n,t}$ is the noisy latent at diffusion timestep $t$. While this gradient update ensures high-fidelity chunk generation, it lacks explicit penalties for the structural drift that accumulates over long sequences.

\noindent\textbf{Latent Cycle-Consistency Objective.} 
To bound this generative drift, we implement the proposed Intrinsic Latent Reversibility via a reverse-prediction model $R_\phi$. As established, $R_\phi$ operates directly on the latent manifold, tasked with reconstructing the \textit{preceding} \yr{latent} $z_{n-1}$ given the current generation $\hat{z}_n$. 
Formally, let $\hat{z}_n$ be the latent chunk synthesized by the forward generator. The reverse model attempts to predict the immediate history $\tilde{z}_{n-1} = R_\phi(\hat{z}_n)$. The cycle-consistency objective is \yr{defined} as the expected squared Euclidean distance between the \textit{actual autoregressive \yr{conditioning context}} used by the forward model and the \textit{reconstructed \yr{history}} inferred by the reverse model:
\begin{equation}
    \mathcal{L}_{cycle} = \mathbb{E}_{\hat{z}_n \sim G_\theta} \left[ \left\| \hat{z}_{n-1} - R_\phi(\hat{z}_n) \right\|_2^2 \right].
\end{equation}
Minimizing this \yr{objective} explicitly enforces the invertibility assumption (Assumption \ref{assump:reverse_lipschitz}), ensuring that the generated $\hat{z}_n$ retains sufficient causal information to recover its origin\yr{, thereby preventing error accumulation.}

\noindent\textbf{Joint Optimization.} 
The final training objective seamlessly integrates the distribution-level supervision with our structural cycle constraint:
\begin{equation}
    \mathcal{L}_{total} = \mathcal{L}_{fwd} + \lambda \mathcal{L}_{cycle},
\end{equation}
where $\lambda$ is a hyperparameter scaling the penalty strength. 
Crucially, during backpropagation, the gradients from $\mathcal{L}_{cycle}$ flow through the 
reverse model $R_\phi$ and back into the forward generator $G_\theta$. This mechanism implicitly endows $G_\theta$ with \textit{foresight}—it penalizes the generation of physically implausible artifacts (like disappearing objects) that, while locally reasonably under $\mathcal{L}_{fwd}$, fail to accurately reconstruct their history, explicitly pruning divergent trajectories during the learning phase.

\subsection{Cycle-Guided Inference: Optimizing Generative Latents via Runtime Guidance}
\label{sec:inference-method}
While the proposed Cycle-Consistent Learning paradigm ensures the generator intrinsically 
\yr{preserves temporal reversibility,}
it requires full-scale parameter retraining. In the era of large-scale foundation models, such retraining is often computationally prohibitive or practically \yr{infeasible} due to closed-source weights. To address this limitation and extend the benefits of temporal reversibility to broader scenarios, we introduce \textbf{Cycle-Guided Inference (CGI)}, a zero-shot, training-free latent optimization strategy. Crucially, this strategy is model-agnostic: as long as the target model and the reverse corrector \yr{operate within the same latent manifold (}sharing the same Video VAE), CGI can be seamlessly \yr{plugged into} \textit{any} off-the-shelf autoregressive video generator, \yr{enabling} structural hallucinations \yr{mitigation} without updating a single model parameter.

\noindent\textbf{Cycle Guidance in Latent Space.} During the autoregressive inference phase, the weights of both the forward generator $G_\theta$ and the reverse model $R_\phi$ are strictly frozen. While earlier sections abstract the generation of the $n$-th block as a single output $z_n$, actual synthesis in diffusion models involves an iterative denoising process over timesteps $\{t_T, \dots, t_1\}$. Unlike conventional decoding that passively accepts forward predictions, we actively rectify accumulated errors at intermediate diffusion timesteps committing them to the historical context buffer (KV cache). Let $z_{n,t}$ denote the latent state at diffusion timestep $t$. First, the forward generator $G_\theta$ predicts the corresponding clean latent $\hat{z}_{n|t} = G_\theta(z_{n,t}, t, \mathcal{H})$ based on the historical context $\mathcal{H}$. To evaluate the physical plausibility of this prediction, we compute the cycle discrepancy $\mathcal{D}$. Specifically, the reverse evaluation is formulated as a two-stage autoregressive process. The predicted clean latent is temporally flipped, denoted by the operator $\mathcal{F}(\cdot)$, and processed by the reverse model at a fixed context noise level $t_{\text{ctx}}$ to construct a reverse contextual cache $\mathcal{H}_{\text{rev}}$. Subsequently, the reverse model utilizes this newly constructed cache to predict the clean predecessor state from standard Gaussian noise $\epsilon$, conditioned on the initial diffusion timestep $t_T$. The discrepancy is defined as the Euclidean distance between the causal condition $\hat{z}_{n-1}$ (the confirmed output of the previous block) and the time-flipped reverse reconstruction:

\begin{equation}
\begin{gathered}
    R_\phi(\mathcal{F}(\hat{z}_{n|t}^{(k)}), t_{\text{ctx}},\mathcal{H}_{\text{rev}}^{(k)}), \quad
    \tilde{z}_{n-1}^{(k)} = \mathcal{F}\left( R_\phi(\epsilon, t_T, \mathcal{H}_{\text{rev}}^{(k)}) \right),\\
    \mathcal{D}(z_{n,t}^{(k)}) = \left\| \hat{z}_{n-1} - \tilde{z}_{n-1}^{(k)} \right\|_2^2,
\end{gathered}
\end{equation}
where $\hat{z}_{n|t}^{(k)} = G_\theta(z_{n,t}^{(k)}, t, \mathcal{H})$, $\epsilon \sim \mathcal{N}(0, \mathbf{I})$, and $k$ is optimization iteration index.

\noindent\textbf{Gradient-Based Latent Refinement.} Since forward prediction and reverse reconstruction are fully differentiable processes, we can iteratively refine the latent state via cycle guidance. To maximize efficiency and structural impact, this optimization is strategically applied only during a specific window of early denoising timesteps (from $T_{\text{start}}$ to $T_{\text{end}}$). We compute the gradient of the cycle discrepancy with respect to the current state $z_{n,t}^{(k)}$ and perform gradient descent:
\begin{equation}
    z_{n,t}^{(k+1)} = z_{n,t}^{(k)} - \eta \nabla_{z_{n,t}^{(k)}} \mathcal{D}(z_{n,t}^{(k)}),
\end{equation}
where $\eta$ is the optimization step size. After $K$ iterations, the refined state $z_{n,t}^{(K)}$ is detached from the computation graph. This optimized state is then passed to the standard diffusion transition function $\Psi$ to obtain the latent state for the subsequent timestep. Upon reaching the final denoising step $t_1$, the resulting clean latent $\hat{z}_n$ is appended to the historical context buffer to guide the generation of the next block. This active rectification mechanism acts as a structural bottleneck, preventing inherent drift from manifesting as visual degradation. The complete inference procedure is summarized in Supplementary Material.

\section{Experiments}
\label{sec:experiment}
\subsection{Experiment Settings}
\noindent\textbf{Baselines.} We evaluate the proposed method against several video generation baselines, categorized by their architectural paradigms. For bidirectional diffusion and transformer models, we compare with LTX-Video~\cite{ltxvideo-HaCohen2024LTX} and Wan2.1~\cite{wan-2025wan}. Within the autoregressive family, we evaluate general-purpose models of varying scales, including NOVA~\cite{nova-deng2024autoregressive} (0.6B), SkyReels-V2~\cite{Skyreels-v2-chen2025skyreels} (1.3B), Pyramid Flow~\cite{Pyramid-Flow-jin2024pyramidal} (2B), and MAGI-1~\cite{magi-1-dobrosotskaya2000magi} (4.5B). To ensure a direct comparison with models sharing our foundational architecture and efficient training setup, we include the 1.3B parameter distilled few-step generators CausVid~\cite{causvid-yin2025slow} and chunk-wise Self Forcing~\cite{selfforcing-huang2025self}. Finally, to assess performance over extended sequences, we compare against models designed for long video generation: LongLive~\cite{longlive-yang2025longlive}, Self Forcing++~\cite{selfforcingplusplus-cui2025self}, Rolling Forcing~\cite{Rolling-forcingliu2025rolling}, Infinity-RoPE~\cite{Infinity-rope-yesiltepe2025infinity}, and Context Forcing~\cite{contextforcing-chen2026context}.

\noindent\textbf{Evaluation Metrics.} We report the performance on VBench~\cite{vbench-huang2023vbench,vbench2-zheng2025vbench2} following \cite{selfforcing-huang2025self,longlive-yang2025longlive,contextforcing-chen2026context}. To assess physical consistency, we use two metrics. The first is Physical Commonsense (PC) from the VideoPhy benchmark~\cite{videophy-bansal2024videophy,videophy2-bansal2025videophy}, which measures adherence to real-world physical laws. The second is the Physical Alignment and Consistency Evaluation (PACE), an LLM-as-a-Judge metric powered by Gemini. PACE scores videos from 0 to 100, penalizing physical hallucinations based on prompt compliance and four criteria: gravity and mass representation, collision dynamics, motion continuity (avoiding sudden teleportation or unnatural warping), and long-term temporal coherence.

\begin{table}[t]
\centering
\caption{Quantitative comparison on the VBench for 5s and 60s video generation.}
\resizebox{0.8\textwidth}{!}{
\begin{tabular}{lc ccc ccc}
\toprule
\multirow{2}{*}{\textbf{Model}} & \multirow{2}{*}{\textbf{\#Params}} & \multicolumn{3}{c}{\textbf{Evaluation scores on 5s $\uparrow$}} & \multicolumn{3}{c}{\textbf{Evaluation scores on 60s $\uparrow$}} \\
\cmidrule(lr){3-5} \cmidrule(lr){6-8}
& & \textbf{Total} & \textbf{Quality} & \textbf{Semantic} & \textbf{Total} & \textbf{Quality} & \textbf{Semantic} \\
\midrule
\multicolumn{8}{l}{\textit{Bidirectional models}} \\
\midrule
LTX-Video~\cite{ltxvideo-HaCohen2024LTX} & 1.9B & 80.00 & 82.30 & 70.79 & - & - & - \\
Wan2.1~\cite{wan-2025wan} & 1.3B & 84.26 & 85.30 & 80.09 & - & - & - \\
\midrule
\multicolumn{8}{l}{\textit{Autoregressive models}} \\
\midrule
SkyReels-V2~\cite{Skyreels-v2-chen2025skyreels} & 1.3B & 82.67 & 84.70 & 74.53 & 70.47 & 75.30 & 51.15 \\
MAGI-1~\cite{magi-1-dobrosotskaya2000magi} & 4.5B & 79.18 & 82.04 & 67.74 & 69.87 & 76.12 & 44.87 \\
CausVid~\cite{causvid-yin2025slow} & 1.3B & 81.20 & 84.05 & 69.80 & 71.04 & 76.80 & 48.01 \\
NOVA~\cite{nova-deng2024autoregressive} & 0.6B & 80.12 & 80.39 & 79.05 & 65.25 & 70.25 & 45.24 \\
Pyramid Flow~\cite{Pyramid-Flow-jin2024pyramidal} & 2B & 81.72 & 84.74 & 69.62 & - & - & - \\
Self Forcing, chunk-wise~\cite{selfforcing-huang2025self} & 1.3B & 84.31 & 85.07 & 81.28 & 71.86 & 77.20 & 50.51 \\
\midrule
\multicolumn{8}{l}{\textit{Long autoregressive models}} \\
\midrule
LongLive~\cite{longlive-yang2025longlive} & 1.3B & 83.72 & 85.42 & 76.95 & 82.62 & 84.53 & 74.97 \\
Self Forcing++~\cite{selfforcingplusplus-cui2025self} & 1.3B & 83.11 & 83.79 & 80.37 & - & - & - \\
Rolling Forcing~\cite{Rolling-forcingliu2025rolling} & 1.3B & 81.22 & 84.08 & 69.78 & 79.31 & 81.87 & 67.69 \\
Infinity-RoPE~\cite{Infinity-rope-yesiltepe2025infinity} & 1.3B & 81.79 & 83.27 & 75.87 & 79.99 & 80.81 & 74.30 \\
Context Forcing~\cite{contextforcing-chen2026context} & 1.3B & 83.44 & 84.98 & 77.29 & 82.45 & 83.55 & 76.10 \\
\rowcolor{CornflowerBlue!20}
\textbf{Ours} & 1.3B &\textbf{84.36} & \textbf{86.14} & 77.25 & \textbf{82.88} & 84.39 & \textbf{76.86}\\

\bottomrule
\end{tabular}
\label{tab:quantitative comparison}
}
\end{table}

\subsection{Comparison Results}
\textbf{Quantitative Evaluation on VBench.} We comprehensively evaluate our method against state-of-the-art bidirectional, standard autoregressive, and long-video specific autoregressive models on the VBench benchmark. As reported in Table \ref{tab:quantitative comparison}, we evaluate both short-horizon (5s) and long-horizon (60s) video generation to demonstrate our model's robustness against generative drift. 

For 5-second generation, our model achieves the highest Total score and Quality score, outperforming strong baselines such as LongLive and Self-Forcing. The superiority of our approach becomes overwhelmingly evident in the extremely long-horizon (60s) setting. Standard autoregressive models suffer from catastrophic error accumulation over extended contexts; for instance, the Total score of Self-Forcing plummets from 84.31 (5s) to 71.86 (60s), and SkyReels-V2 drops from 82.67 to 70.47. In stark contrast, our method effectively bounds this generative drift, maintaining a remarkable Total score and achieving the highest Semantic score at 60 seconds. This minimal performance degradation over time confirms that our cycle-consistency framework successfully preserves the structural integrity and visual fidelity of the generated sequence over infinite horizons.

\textbf{Physical Consistency Evaluation.} To further validate whether our model accurately captures the underlying physical rules of the visual world, we evaluate it on Physical Commonsense (PC) and  Physical Alignment and Consistency Evaluation (PACE) metrics. As shown in Table \ref{tab:physical_consistency}, our method significantly outperforms all baseline approaches, achieving the highest average score of 75.66. By enforcing temporal reversibility, our model inherently internalizes causal constraints, endowing it with a superior understanding of complex physical dynamics compared to standard autoregressive predictors.

\textbf{Qualitative Comparison.} Figure \ref{fig:qualitative_comparison} visualizes the 60-second generation quality of our model compared to the baselines. As the autoregressive steps accumulate, the baseline methods exhibit severe structural distortion, loss of the main subject, and prominent identity shifts. Our Cycle-World framework, however, acts as a strict temporal regularizer. It consistently preserves the subject's identity, maintains sharp background details, and ensures temporal continuity from the first frame to the very last, translating the quantitative resilience observed in Table \ref{tab:quantitative comparison} into striking visual stability.

\begin{figure}[t]
    \centering
    \includegraphics[width=\textwidth]{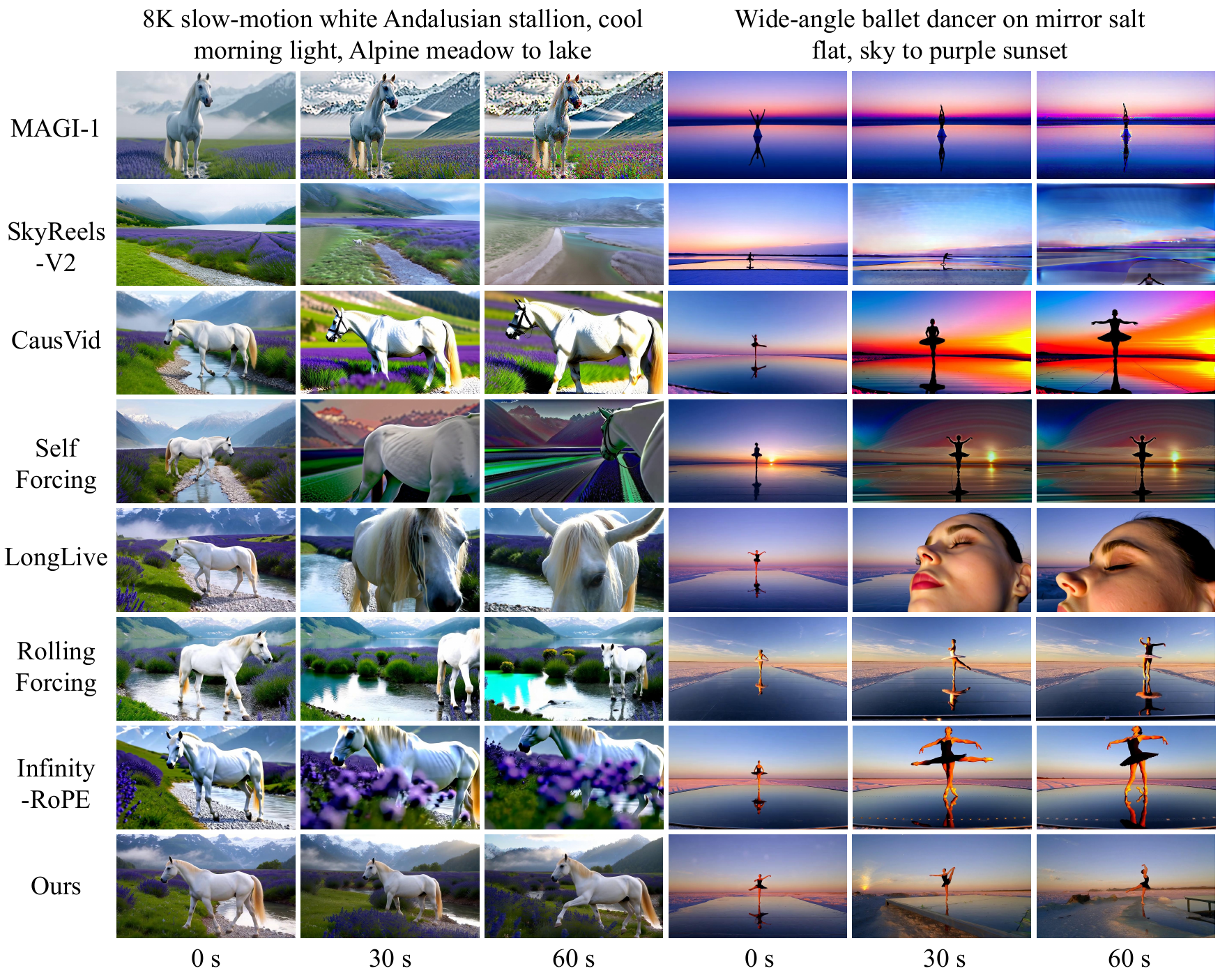}
    \caption{Qualitative comparison of long video generation. We compare our proposed method against several baseline models. The figure displays sampled frames at 0, 30, and 60 seconds for two distinct scenes. While the baseline methods struggle with subject loss, severe structural distortion, or identity shifts over the extended timeframe, our method successfully preserves temporal consistency, visual quality, and subject identity throughout the entire 60-second duration.}
    \label{fig:qualitative_comparison}
\end{figure}
\begin{table*}[htbp]
  \centering
  \begin{minipage}[t]{0.40\textwidth} 
    \centering
    \caption{Quantitative evaluation of physical consistency. Our proposed method outperforms all baseline approaches across both metrics, demonstrating a comprehensive understanding of complex physical dynamics.}
    \label{tab:physical_consistency}
    \begin{tabular}{lccc}
      \toprule
      \textbf{Method} & \textbf{PC} & \textbf{PACE} & \textbf{Avg.} \\ 
      \midrule
      Self-Forcing & $55.97$ & $78.57$ & $67.27$ \\
      LongLive & $61.94$ & $78.03$ & $69.99$ \\
      RollingForcing & $67.91$ & $75.05$ & $71.48$ \\
      Infinity-Rope & $60.45$ & $78.58$ & $69.52$ \\
      \midrule
      \rowcolor{CornflowerBlue!20}
      \textbf{Ours} & \textbf{$69.40$} & \textbf{$81.91$} &\textbf{$75.66$} \\
      \bottomrule
    \end{tabular}
  \end{minipage}
  \hfill
  \begin{minipage}[t]{0.57\textwidth} 
    \centering
    \caption{Ablation study of the proposed components on the VBench benchmark. We evaluate the individual and synergistic effects of the training-time cycle loss ($\mathcal{L}_{cycle}$) and the inference-time cycle guidance on 5s video generation. The best results are highlighted in \textbf{bold}.}
    \label{tab:ablation}
    \begin{tabular}{lccccc}
      \toprule
      \textbf{Method} & \textbf{Tot.} & \textbf{Qual.} & \textbf{Sem.} & \textbf{PC} & \textbf{PACE} \\ 
      \midrule
      Baseline  & 83.72 & 85.42 & 76.95 & 61.94 & 78.03 \\
      + CCL & 83.89 & 85.72 & 76.55 & 68.70 & 79.5 \\
      + CGI  & \textbf{84.51} & \textbf{86.37} & \textbf{77.05} & 65.67 & 78.80 \\
      \midrule
      \rowcolor{CornflowerBlue!20}
      \textbf{Cycle-World}      & 84.36 & 86.14 & 77.25 & \textbf{69.40} & \textbf{81.91} \\
      \bottomrule
    \end{tabular}
  \end{minipage}
\end{table*}
\subsection{Ablation Studies}
\begin{figure}[t]
    \centering
    \includegraphics[width=\textwidth]{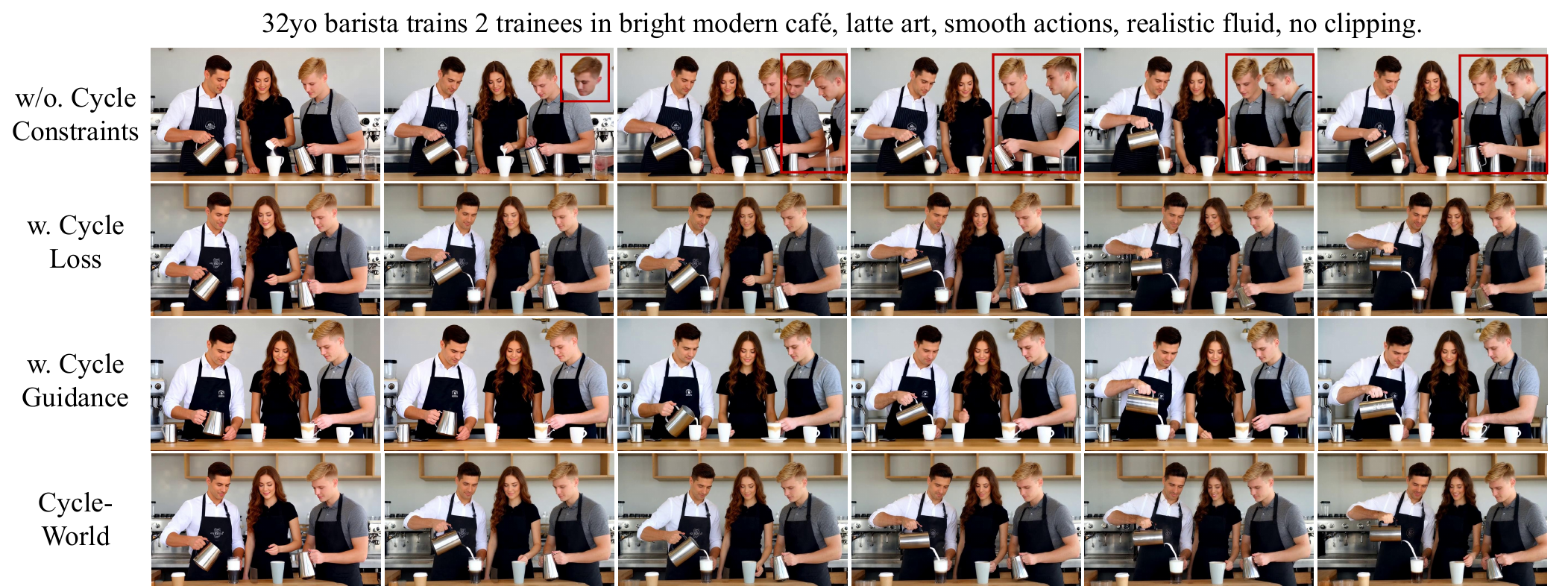}
\caption{\textbf{Qualitative ablation study of the proposed Cycle-World components.} We compare the visual quality and temporal consistency of different model variants. While the training cycle loss ($\mathcal{L}_{cycle}$) establishes a robust parametric foundation for structural stability, and the runtime corrector acts as an active safeguard against temporal artifacts, their combination achieves a powerful dual-phase synergy. The full Cycle-World model effectively prevents trajectory drift and background degradation, yielding superior results in long video generation.}
    \label{fig:ablation}
\end{figure}
\subsubsection{Effectiveness of Cycle-World Components.}
To evaluate the contributions of our proposed modules, we conduct an ablation study on visual quality and physical consistency. Table \ref{tab:ablation} reports the quantitative metrics on 5-second video generation, while Figure \ref{fig:ablation} illustrates the qualitative long-horizon stability over 30 seconds.

\noindent\textbf{Quantitative Synergy in Short-Horizon Generation.} As shown in Table \ref{tab:ablation}, the Baseline achieves a Total VBench score of 83.72 while scoring 61.94 on PC. Adding the Cycle-Consistent Learning (+ CCL) provides a learned parametric prior. While it moderately improves general visual metrics, its primary benefit is in physical consistency, raising PC by nearly 7 points. This suggests that penalizing reverse-prediction errors embeds causal constraints into the generator. Conversely, applying the runtime corrector solely during inference (+ CGI) acts as an instance-level regularizer against visual degradation, achieving the highest VBench Quality and Total scores. However, relying purely on gradient-based inference optimization without a learned parametric physical prior yields sub-optimal physical consistency.

The full Cycle-World model combines both mechanisms. While its general VBench scores are marginally lower than the CGI-only variant, it achieves the highest performance in physical consistency. This combined strategy ensures the generated videos maintain both visual quality and physical accuracy.

\noindent\textbf{Qualitative Results in Long-Horizon Generation.} The benefit of this combination is more evident when extending the generation to 30 seconds (Fig. \ref{fig:ablation}). In long-horizon synthesis, the Baseline exhibits rapid trajectory drift and background degradation. The +CCL variant preserves the core structural identity but struggles to suppress high-frequency temporal artifacts over extended autoregressive steps. The +CGI variant maintains aesthetic coherence locally but eventually drifts from the physical manifold due to the lack of an intrinsic causal prior. The full Cycle-World model combines these strengths. By using CCL to keep initial forward proposals close to the reversible manifold, the runtime corrector (CGI) performs more accurate latent refinement. This approach eliminates spatial distortion and temporal identity shifts, producing visually stable and physically grounded long videos.
\section{Conclusion}
\label{sec:conclusion}

In this paper, we presented Cycle-World, a novel autoregressive video generation framework designed to tackle the pervasive issue of unconstrained generative drift in long-horizon synthesis. Based on the theoretical insight that forward prediction errors can be strictly bottlenecked by enforcing temporal reversibility, we proposed a unified strategy that maintains causal consistency across both the training and inference phases. 
During training, we introduce a reverse-prediction cycle loss alongside the Distribution Matching Distillation (DMD) objective. This explicitly embeds causal constraints into the generator's parametric weights, establishing a robust foundation for structurally stable frame generation. During inference, we repurpose the frozen reverse model as a fully differentiable runtime corrector. By leveraging gradient-based cycle guidance, this mechanism acts as an active, instance-level safeguard, dynamically pruning out-of-distribution trajectory drift and temporal artifacts before they compound. 
Extensive experiments on the VBench benchmark demonstrate that this dual-phase synergy effectively prevents the structural collapse and aesthetic degradation typically observed in prolonged autoregressive generation. Consequently, Cycle-World achieves state-of-the-art performance in producing highly consistent, smooth, and high-fidelity 60-second videos.

\section*{Acknowledgement}
This work was supported by National Natural Science Foundation of China (No. 62302297, 625B2115, 62272447, 62472285, 72192821, 62472285), the Fundamental Research Funds for the Central Universities (YG2023QNB17, YG2024QNA44). [Re Prof Tao] This project is supported by the National Research Foundation, Singapore, under its NRF Professorship Award No. NRF-P2024-001.


%
%

\bibliographystyle{splncs04}
\bibliography{main}

@String(CVPR  = {IEEE Conf. Comput. Vis. Pattern Recog.})

@String(CVPR  = {CVPR})

@article{ddpm-ho2020denoising,
  title={Denoising diffusion probabilistic models},
  author={Ho, Jonathan and Jain, Ajay and Abbeel, Pieter},
  journal={Advances in neural information processing systems},
  volume={33},
  pages={6840--6851},
  year={2020}
}

@inproceedings{iddpm-nichol2021improved,
  title={Improved denoising diffusion probabilistic models},
  author={Nichol, Alexander Quinn and Dhariwal, Prafulla},
  booktitle={International conference on machine learning},
  pages={8162--8171},
  year={2021},
  organization={PMLR}
}

@article{ddim-song2020denoising,
  title={Denoising diffusion implicit models},
  author={Song, Jiaming and Meng, Chenlin and Ermon, Stefano},
  journal={arXiv preprint arXiv:2010.02502},
  year={2020}
}

@inproceedings{dit-peebles2023scalable,
  title={Scalable diffusion models with transformers},
  author={Peebles, William and Xie, Saining},
  booktitle={Proceedings of the IEEE/CVF international conference on computer vision},
  pages={4195--4205},
  year={2023}
}

@article{wan-2025wan,
  title={Wan: Open and advanced large-scale video generative models},
  author={Wan, Team and Wang, Ang and Ai, Baole and Wen, Bin and Mao, Chaojie and Xie, Chen-Wei and Chen, Di and Yu, Feiwu and Zhao, Haiming and Yang, Jianxiao and others},
  journal={arXiv preprint arXiv:2503.20314},
  year={2025}
}

@article{hunyuanvideo-kong2024hunyuanvideo,
  title={Hunyuanvideo: A systematic framework for large video generative models},
  author={Kong, Weijie and Tian, Qi and Zhang, Zijian and Min, Rox and Dai, Zuozhuo and Zhou, Jin and Xiong, Jiangfeng and Li, Xin and Wu, Bo and Zhang, Jianwei and others},
  journal={arXiv preprint arXiv:2412.03603},
  year={2024}
}

@article{Open-sora-zheng2024open,
  title={Open-sora: Democratizing efficient video production for all},
  author={Zheng, Zangwei and Peng, Xiangyu and Yang, Tianji and Shen, Chenhui and Li, Shenggui and Liu, Hongxin and Zhou, Yukun and Li, Tianyi and You, Yang},
  journal={arXiv preprint arXiv:2412.20404},
  year={2024}
}

@article{cogvideox-yang2024cogvideox,
  title={Cogvideox: Text-to-video diffusion models with an expert transformer},
  author={Yang, Zhuoyi and Teng, Jiayan and Zheng, Wendi and Ding, Ming and Huang, Shiyu and Xu, Jiazheng and Yang, Yuanming and Hong, Wenyi and Zhang, Xiaohan and Feng, Guanyu and others},
  journal={arXiv preprint arXiv:2408.06072},
  year={2024}
}

@article{cogvideo-hong2022cogvideo,
  title={Cogvideo: Large-scale pretraining for text-to-video generation via transformers},
  author={Hong, Wenyi and Ding, Ming and Zheng, Wendi and Liu, Xinghan and Tang, Jie},
  journal={arXiv preprint arXiv:2205.15868},
  year={2022}
}

@article{videopoet-kondratyuk2023videopoet,
  title={Videopoet: A large language model for zero-shot video generation},
  author={Kondratyuk, Dan and Yu, Lijun and Gu, Xiuye and Lezama, Jos{\'e} and Huang, Jonathan and Schindler, Grant and Hornung, Rachel and Birodkar, Vighnesh and Yan, Jimmy and Chiu, Ming-Chang and others},
  journal={arXiv preprint arXiv:2312.14125},
  year={2023}
}

@article{magi-1-dobrosotskaya2000magi,
  title={MAGI-1 interacts with $\beta$-catenin and is associated with cell--cell adhesion structures},
  author={Dobrosotskaya, Irina Y and James, Guy L},
  journal={Biochemical and biophysical research communications},
  volume={270},
  number={3},
  pages={903--909},
  year={2000},
  publisher={Elsevier}
}

@article{Pyramid-Flow-jin2024pyramidal,
  title={Pyramidal flow matching for efficient video generative modeling},
  author={Jin, Yang and Sun, Zhicheng and Li, Ningyuan and Xu, Kun and Jiang, Hao and Zhuang, Nan and Huang, Quzhe and Song, Yang and Mu, Yadong and Lin, Zhouchen},
  journal={arXiv preprint arXiv:2410.05954},
  year={2024}
}

@article{nova-deng2024autoregressive,
  title={Autoregressive video generation without vector quantization},
  author={Deng, Haoge and Pan, Ting and Diao, Haiwen and Luo, Zhengxiong and Cui, Yufeng and Lu, Huchuan and Shan, Shiguang and Qi, Yonggang and Wang, Xinlong},
  journal={arXiv preprint arXiv:2412.14169},
  year={2024}
}

@inproceedings{causvid-yin2025slow,
  title={From slow bidirectional to fast autoregressive video diffusion models},
  author={Yin, Tianwei and Zhang, Qiang and Zhang, Richard and Freeman, William T and Durand, Fredo and Shechtman, Eli and Huang, Xun},
  booktitle={Proceedings of the IEEE/CVF Conference on Computer Vision and Pattern Recognition},
  pages={22963--22974},
  year={2025}
}

@article{selfforcing-huang2025self,
  title={Self forcing: Bridging the train-test gap in autoregressive video diffusion},
  author={Huang, Xun and Li, Zhengqi and He, Guande and Zhou, Mingyuan and Shechtman, Eli},
  journal={arXiv preprint arXiv:2506.08009},
  year={2025}
}

@inproceedings{Seine-chen2023seine,
  title={Seine: Short-to-long video diffusion model for generative transition and prediction},
  author={Chen, Xinyuan and Wang, Yaohui and Zhang, Lingjun and Zhuang, Shaobin and Ma, Xin and Yu, Jiashuo and Wang, Yali and Lin, Dahua and Qiao, Yu and Liu, Ziwei},
  booktitle={The Twelfth International Conference on Learning Representations},
  year={2023}
}

@article{Gen-l-video-wang2023gen,
  title={Gen-l-video: Multi-text to long video generation via temporal co-denoising},
  author={Wang, Fu-Yun and Chen, Wenshuo and Song, Guanglu and Ye, Han-Jia and Liu, Yu and Li, Hongsheng},
  journal={arXiv preprint arXiv:2305.18264},
  year={2023}
}

@article{freenoise-qiu2023freenoise,
  title={Freenoise: Tuning-free longer video diffusion via noise rescheduling},
  author={Qiu, Haonan and Xia, Menghan and Zhang, Yong and He, Yingqing and Wang, Xintao and Shan, Ying and Liu, Ziwei},
  journal={arXiv preprint arXiv:2310.15169},
  year={2023}
}

@inproceedings{Nuwa-xl-yin2023nuwa,
  title={Nuwa-xl: Diffusion over diffusion for extremely long video generation},
  author={Yin, Shengming and Wu, Chenfei and Yang, Huan and Wang, Jianfeng and Wang, Xiaodong and Ni, Minheng and Yang, Zhengyuan and Li, Linjie and Liu, Shuguang and Yang, Fan and others},
  booktitle={Proceedings of the 61st Annual Meeting of the Association for Computational Linguistics (Volume 1: Long Papers)},
  pages={1309--1320},
  year={2023}
}

@inproceedings{Align-your-latents-blattmann2023align,
  title={Align your latents: High-resolution video synthesis with latent diffusion models},
  author={Blattmann, Andreas and Rombach, Robin and Ling, Huan and Dockhorn, Tim and Kim, Seung Wook and Fidler, Sanja and Kreis, Karsten},
  booktitle={Proceedings of the IEEE/CVF conference on computer vision and pattern recognition},
  pages={22563--22575},
  year={2023}
}

@article{Skyreels-v2-chen2025skyreels,
  title={Skyreels-v2: Infinite-length film generative model},
  author={Chen, Guibin and Lin, Dixuan and Yang, Jiangping and Lin, Chunze and Zhu, Junchen and Fan, Mingyuan and Zhang, Hao and Chen, Sheng and Chen, Zheng and Ma, Chengcheng and others},
  journal={arXiv preprint arXiv:2504.13074},
  year={2025}
}

@article{Diffusion-forcing-chen2024diffusion,
  title={Diffusion forcing: Next-token prediction meets full-sequence diffusion},
  author={Chen, Boyuan and Mart{\'\i} Mons{\'o}, Diego and Du, Yilun and Simchowitz, Max and Tedrake, Russ and Sitzmann, Vincent},
  journal={Advances in Neural Information Processing Systems},
  volume={37},
  pages={24081--24125},
  year={2024}
}

@article{longlive-yang2025longlive,
  title={Longlive: Real-time interactive long video generation},
  author={Yang, Shuai and Huang, Wei and Chu, Ruihang and Xiao, Yicheng and Zhao, Yuyang and Wang, Xianbang and Li, Muyang and Xie, Enze and Chen, Yingcong and Lu, Yao and others},
  journal={arXiv preprint arXiv:2509.22622},
  year={2025}
}

@article{dmd2-yin2024improved,
  title={Improved distribution matching distillation for fast image synthesis},
  author={Yin, Tianwei and Gharbi, Micha{\"e}l and Park, Taesung and Zhang, Richard and Shechtman, Eli and Durand, Fredo and Freeman, Bill},
  journal={Advances in neural information processing systems},
  volume={37},
  pages={47455--47487},
  year={2024}
}

@inproceedings{dmd-yin2024one,
  title={One-step diffusion with distribution matching distillation},
  author={Yin, Tianwei and Gharbi, Micha{\"e}l and Zhang, Richard and Shechtman, Eli and Durand, Fredo and Freeman, William T and Park, Taesung},
  booktitle={Proceedings of the IEEE/CVF conference on computer vision and pattern recognition},
  pages={6613--6623},
  year={2024}
}

@inproceedings{TimeGrad-rasul2021autoregressive,
  title={Autoregressive denoising diffusion models for multivariate probabilistic time series forecasting},
  author={Rasul, Kashif and Seward, Calvin and Schuster, Ingmar and Vollgraf, Roland},
  booktitle={International conference on machine learning},
  pages={8857--8868},
  year={2021},
  organization={PMLR}
}

@article{teacher-forcing-williams1989learning,
  title={A learning algorithm for continually running fully recurrent neural networks},
  author={Williams, Ronald J and Zipser, David},
  journal={Neural computation},
  volume={1},
  number={2},
  pages={270--280},
  year={1989},
  publisher={MIT Press}
}

@misc{sora-openai2024sora,
  author       = {OpenAI},
  title        = {Sora},
  year         = {2024},
  howpublished = {\url{https://openai.com/sora}},
}

@misc{sora2openai2025sora2,
  author       = {OpenAI},
  title        = {Sora 2},
  year         = {2025},
  howpublished = {\url{https://openai.com/index/sora-2/}},
}

@article{seedance-gao2025seedance,
  title={Seedance 1.0: Exploring the boundaries of video generation models},
  author={Gao, Yu and Guo, Haoyuan and Hoang, Tuyen and Huang, Weilin and Jiang, Lu and Kong, Fangyuan and Li, Huixia and Li, Jiashi and Li, Liang and Li, Xiaojie and others},
  journal={arXiv preprint arXiv:2506.09113},
  year={2025}
}

@misc{seedance2,
  author       = {ByteDance Seed},
  title        = {Seedance 2.0},
  year         = {2026},
  howpublished = {\url{https://seed.bytedance.com/en/seedance2_0}},
}

@misc{kling,
  author       = {Kuaishou Technology},
  title        = {Kling},
  year         = {2025},
  howpublished = {\url{https://kling.kuaishou.com/}},
}

@article{Rolling-forcingliu2025rolling,
  title={Rolling forcing: Autoregressive long video diffusion in real time},
  author={Liu, Kunhao and Hu, Wenbo and Xu, Jiale and Shan, Ying and Lu, Shijian},
  journal={arXiv preprint arXiv:2509.25161},
  year={2025}
}

@article{Infinity-rope-yesiltepe2025infinity,
  title={Infinity-rope: Action-controllable infinite video generation emerges from autoregressive self-rollout},
  author={Yesiltepe, Hidir and Meral, Tuna Han Salih and Akan, Adil Kaan and Oktay, Kaan and Yanardag, Pinar},
  journal={arXiv preprint arXiv:2511.20649},
  year={2025}
}

@article{ltxvideo-HaCohen2024LTX,
  title={LTX-Video: Realtime Video Latent Diffusion},
  author={HaCohen, Yoav and Chiprut, Nisan and Brazowski, Benny and Shalem, Daniel and Moshe, Dudu and Richardson, Eitan and Levin, Eran and Shiran, Guy and Zabari, Nir and Gordon, Ori and Panet, Poriya and Weissbuch, Sapir and Kulikov, Victor and Bitterman, Yaki and Melumian, Zeev and Bibi, Ofir},
  journal={arXiv preprint arXiv:2501.00103},
  year={2024}
}

@article{contextforcing-chen2026context,
  title={Context Forcing: Consistent Autoregressive Video Generation with Long Context},
  author={Chen, Shuo and Wei, Cong and Sun, Sun and Nie, Ping and Zhou, Kai and Zhang, Ge and Yang, Ming-Hsuan and Chen, Wenhu},
  journal={arXiv preprint arXiv:2602.06028},
  year={2026}
}

@article{selfforcingplusplus-cui2025self,
  title={Self-forcing++: Towards minute-scale high-quality video generation},
  author={Cui, Justin and Wu, Jie and Li, Ming and Yang, Tao and Li, Xiaojie and Wang, Rui and Bai, Andrew and Ban, Yuanhao and Hsieh, Cho-Jui},
  journal={arXiv preprint arXiv:2510.02283},
  year={2025}
}

@InProceedings{vbench-huang2023vbench,
     title={{VBench}: Comprehensive Benchmark Suite for Video Generative Models},
     author={Huang, Ziqi and He, Yinan and Yu, Jiashuo and Zhang, Fan and Si, Chenyang and Jiang, Yuming and Zhang, Yuanhan and Wu, Tianxing and Jin, Qingyang and Chanpaisit, Nattapol and Wang, Yaohui and Chen, Xinyuan and Wang, Limin and Lin, Dahua and Qiao, Yu and Liu, Ziwei},
     booktitle={Proceedings of the IEEE/CVF Conference on Computer Vision and Pattern Recognition},
     year={2024}
 }

@article{vbench2-zheng2025vbench2,
     title={{VBench-2.0}: Advancing Video Generation Benchmark Suite for Intrinsic Faithfulness},
     author={Zheng, Dian and Huang, Ziqi and Liu, Hongbo and Zou, Kai and He, Yinan and Zhang, Fan and Zhang, Yuanhan and He, Jingwen and Zheng, Wei-Shi and Qiao, Yu and Liu, Ziwei},
     journal={arXiv preprint arXiv:2503.21755},
     year={2025}
 }

@article{videophy-bansal2024videophy,
  title={Videophy: Evaluating physical commonsense for video generation},
  author={Bansal, Hritik and Lin, Zongyu and Xie, Tianyi and Zong, Zeshun and Yarom, Michal and Bitton, Yonatan and Jiang, Chenfanfu and Sun, Yizhou and Chang, Kai-Wei and Grover, Aditya},
  journal={arXiv preprint arXiv:2406.03520},
  year={2024}
}

@article{videophy2-bansal2025videophy,
  title={Videophy-2: A challenging action-centric physical commonsense evaluation in video generation},
  author={Bansal, Hritik and Peng, Clark and Bitton, Yonatan and Goldenberg, Roman and Grover, Aditya and Chang, Kai-Wei},
  journal={arXiv preprint arXiv:2503.06800},
  year={2025}
}

@article{vidprom-wang2024vidprom,
  title={VidProM: A Million-scale Real Prompt-Gallery Dataset for Text-to-Video Diffusion Models},
  author={Wang, Wenhao and Yang, Yi},
  booktitle={Thirty-eighth Conference on Neural Information Processing Systems},
  year={2024},
  url={https://openreview.net/forum?id=pYNl76onJL}
}

@inproceedings{zhu2017unpaired,
  title={Unpaired image-to-image translation using cycle-consistent adversarial networks},
  author={Zhu, Jun-Yan and Park, Taesung and Isola, Phillip and Efros, Alexei A},
  booktitle={Proceedings of the IEEE international conference on computer vision},
  pages={2223--2232},
  year={2017}
}

@ARTICLE{Full-Frame-Video-Stabilization,
  author={Karacan, Levent and Sarıgül, Mehmet},
  journal={Computational Visual Media}, 
  title={Full-Frame Video Stabilization via Spatiotemporal Transformers}, 
  year={2025},
  volume={11},
  number={3},
  pages={655-667},
  doi={10.26599/CVM.2025.9450416}}

@ARTICLE{Temporally-consistent-video-colorization,
  author={Liu, Yihao and Zhao, Hengyuan and Chan, Kelvin C. K. and Wang, Xintao and Loy, Chen Change and Qiao, Yu and Dong, Chao},
  journal={Computational Visual Media}, 
  title={Temporally consistent video colorization with deep feature propagation and self-regularization learning}, 
  year={2024},
  volume={10},
  number={2},
  pages={375-395},
  doi={10.1007/s41095-023-0342-8}}

@ARTICLE{a-causal-cnn,
  author={Hou, Shuaiying and Wang, Congyi and Zhuang, Wenlin and Chen, Yu and Wang, Yangang and Bao, Hujun and Chai, Jinxiang and Xu, Weiwei},
  journal={Computational Visual Media}, 
  title={A causal convolutional neural network for multi-subject motion modeling and generation}, 
  year={2024},
  volume={10},
  number={1},
  pages={45-59},
  doi={10.1007/s41095-022-0307-3}}

@article{zhu2026causalforcing,
  title={Causal Forcing: Autoregressive Diffusion Distillation Done Right for High-Quality Real-Time Interactive Video Generation},
  author={Zhu, Hongzhou and Zhao, Min and He, Guande and Su, Hang and Li, Chongxuan and Zhu, Jun},
  journal={arXiv preprint arXiv:2602.02214},
  year={2026}
}

@misc{MetaWorld,
      title={MetaWorld: Scaling Multi-Agent Video World Model from Single-view Video Data}, 
      author={Teng Hu and Mingchun Lu and Yating Wang and Jiangning Zhang and Jinkun Hao and Ye Pan and Ran Yi and Lizhuang Ma and Dacheng Tao},
      year={2026},
      eprint={2606.02753},
      archivePrefix={arXiv},
      primaryClass={cs.CV},
      url={https://arxiv.org/abs/2606.02753}, 
}

@inproceedings{PolyVivid,
 author = {Hu, Teng and Yu, Zhentao and Zhou, Zhengguang and Zhang, Jiangning and Zhou, Yuan and Lu, Qinglin and Yi, Ran},
 booktitle = {Advances in Neural Information Processing Systems},
 editor = {D. Belgrave and C. Zhang and H. Lin and R. Pascanu and P. Koniusz and M. Ghassemi and N. Chen},
 pages = {49394--49420},
 publisher = {Curran Associates, Inc.},
 title = {PolyVivid: Vivid Multi-Subject Video Generation with Cross-Modal Interaction and Enhancement},
 url = {https://proceedings.neurips.cc/paper_files/paper/2025/file/4683beb6bab325650db13afd05d1a14a-Paper-Conference.pdf},
 volume = {38},
 year = {2025}
}

@misc{hu2025hunyuancustommultimodaldrivenarchitecturecustomized,
      title={HunyuanCustom: A Multimodal-Driven Architecture for Customized Video Generation}, 
      author={Teng Hu and Zhentao Yu and Zhengguang Zhou and Sen Liang and Yuan Zhou and Qin Lin and Qinglin Lu},
      year={2025},
      eprint={2505.04512},
      archivePrefix={arXiv},
      primaryClass={cs.CV},
      url={https://arxiv.org/abs/2505.04512}, 
}

@misc{hu2025ultragenhighresolutionvideogeneration,
      title={UltraGen: High-Resolution Video Generation with Hierarchical Attention}, 
      author={Teng Hu and Jiangning Zhang and Zihan Su and Ran Yi},
      year={2025},
      eprint={2510.18775},
      archivePrefix={arXiv},
      primaryClass={cs.CV},
      url={https://arxiv.org/abs/2510.18775}, 
}

@inproceedings{zhang2026uniavgen,
  title={Uniavgen: Unified audio and video generation with asymmetric cross-modal interactions},
  author={Zhang, Guozhen and Zhou, Zixiang and Hu, Teng and Peng, Ziqiao and Zhang, Youliang and Chen, Yi and Zhou, Yuan and Lu, Qinglin and Wang, Limin},
  booktitle={Proceedings of the IEEE/CVF Conference on Computer Vision and Pattern Recognition},
  pages={1950--1960},
  year={2026}
}

@InProceedings{Harmony,
    author    = {Hu, Teng and Yu, Zhentao and Zhang, Guozhen and Su, Zihan and Zhou, Zhengguang and Zhang, Youliang and Zhou, Yuan and Lu, Qinglin and Yi, Ran},
    title     = {Harmony: Harmonizing Audio and Video Generation through Cross-Task Synergy},
    booktitle = {Proceedings of the IEEE/CVF Conference on Computer Vision and Pattern Recognition (CVPR)},
    month     = {June},
    year      = {2026},
    pages     = {16085-16095}
}

@misc{low2025Ovi,
      title={Ovi: Twin Backbone Cross-Modal Fusion for Audio-Video Generation}, 
      author={Chetwin Low and Weimin Wang and Calder Katyal},
      year={2025},
      eprint={2510.01284},
      archivePrefix={arXiv},
      primaryClass={cs.MM},
      url={https://arxiv.org/abs/2510.01284}, 
}

@inproceedings{liu2026javisdit++,
  title     = {JavisDiT++: Unified Modeling and Optimization for Joint Audio-Video Generation},
  author    = {Liu, Kai and Zheng, Yanhao and Wang, Kai and Wu, Shengqiong and Zhang, Rongjunchen and Luo, Jiebo and Hatzinakos, Dimitrios and Liu, Ziwei and Fei, Hao and Chua, Tat-Seng},
  booktitle = {The Fourteenth International Conference on Learning Representations},
  year      = {2026},
}

@article{wang2025universe,
  title={UniVerse-1: Unified Audio-Video Generation via Stitching of Experts},
  author={Wang, Duomin and Zuo, Wei and Li, Aojie and Chen, Ling-Hao and Liao, Xinyao and Zhou, Deyu and Yin, Zixin and Dai, Xili and Jiang, Daxin and Yu, Gang},
  journal={arXiv preprint arXiv:2509.06155},
  year={2025}
}

@inproceedings{OmniVCus,
 author = {Cai, Yuanhao and Zhang, HE and Chen, Xi and Xing, Jinbo and Hu, Yiwei and Zhou, Yuqian and Zhang, Kai and Zhang, Zhifei and Kim, Soo Ye and Wang, Tianyu and Zhang, Yulun and Yang, Xiaokang and Lin, Zhe and Yuille, Alan},
 booktitle = {Advances in Neural Information Processing Systems},
 editor = {D. Belgrave and C. Zhang and H. Lin and R. Pascanu and P. Koniusz and M. Ghassemi and N. Chen},
 pages = {115404--115423},
 publisher = {Curran Associates, Inc.},
 title = {OmniVCus: Feedforward Subject-driven Video Customization with Multimodal Control Conditions},
 url = {https://proceedings.neurips.cc/paper_files/paper/2025/file/a79054a9da91d73ed3cb1a9e87d7cd2d-Paper-Conference.pdf},
 volume = {38},
 year = {2025}
}

@inproceedings{wang2025lingen,
  title={Lingen: Towards high-resolution minute-length text-to-video generation with linear computational complexity},
  author={Wang, Hongjie and Ma, Chih-Yao and Liu, Yen-Cheng and Hou, Ji and Xu, Tao and Wang, Jialiang and Juefei-Xu, Felix and Luo, Yaqiao and Zhang, Peizhao and Hou, Tingbo and others},
  booktitle={Proceedings of the Computer Vision and Pattern Recognition Conference},
  pages={2578--2588},
  year={2025}
}

@article{hyworld2025,
  title={HY-World 1.5: A Systematic Framework for Interactive World Modeling with Real-Time Latency and Geometric Consistency},
  author={Team HunyuanWorld},
  journal={arXiv preprint},
  year={2025}
}

@misc{2026matrix,
title={Matrix-Game 3.0: Real-Time and Streaming Interactive World Model with Long-Horizon Memory},
author={{Skywork AI Matrix-Game Team}},
year={2026},
howpublished={Technical report},
url={https://github.com/SkyworkAI/Matrix-Game/blob/main/Matrix-Game-3/assets/pdf/report.pdf}
}

@article{mao2025yume,
  title={Yume-1.5: A Text-Controlled Interactive World Generation Model},
  author={Mao, Xiaofeng and Li, Zhen and Li, Chuanhao and Xu, Xiaojie and Ying, Kaining and He, Tong and Pang, Jiangmiao and Qiao, Yu and Zhang, Kaipeng},
  journal={arXiv preprint arXiv:2512.22096},
  year={2025}
}

@misc{hu2026evolutionvideogenerativefoundations,
      title={Evolution of Video Generative Foundations}, 
      author={Teng Hu and Jiangning Zhang and Hongrui Huang and Ran Yi and Zihan Su and Jieyu Weng and Zhucun Xue and Lizhuang Ma and Ming-Hsuan Yang and Dacheng Tao},
      year={2026},
      eprint={2604.06339},
      archivePrefix={arXiv},
      primaryClass={cs.CV},
      url={https://arxiv.org/abs/2604.06339}, 
}
\clearpage


\renewcommand\thefigure{S\arabic{figure}}   
\renewcommand\thetable{S\arabic{table}}     
\renewcommand\theequation{S\arabic{equation}} 

\setcounter{equation}{0}
\setcounter{table}{0}
\setcounter{figure}{0}

\appendix
\label{appendix}

\section{Overview}
\label{sec:overview}

This supplementary material provides further theoretical derivations, algorithmic details, rigorous experimental setups, and extensive qualitative and quantitative evaluations to thoroughly support the claims made in the main manuscript. The document is logically organized as follows:

\begin{itemize}
    \item \textbf{Section \ref{sec:detailed_proofs}} presents the detailed mathematical proofs for the theoretical bounds established in the main text (Theorem 1, Corollary 1, and Proposition 1).
    \item \textbf{Section \ref{sec:algorithm}} provides the complete pseudocode and a step-by-step breakdown of our proposed zero-shot Cycle-Guided Inference (CGI) strategy.
    \item \textbf{Section \ref{sec:implementation_details}} outlines the comprehensive implementation details, including data preparation, model construction, and hyperparameter configurations.
    \item \textbf{Section \ref{sec:eval_metrics_physical}} details the exact definitions and prompt templates used for our Physical Consistency evaluation metrics (PC and PACE).
    \item \textbf{Section \ref{sec:extended_evaluation}} provides extended quantitative comparisons and ablation studies on the challenging 60-second video generation tasks.
    \item \textbf{Section \ref{sec:interactive_generation} \& \ref{sec:single_prompt_generation}} showcase additional qualitative comparisons, highlighting Cycle-World's robustness in interactive generation and single-prompt extended synthesis.
    \item \textbf{Section \ref{sec:hyperparameters}} presents a thorough sensitivity analysis of key training and inference hyperparameters.
    \item \textbf{Section \ref{sec:plug_and_play}} further validates the plug-and-play extensibility of CGI on standard forward-only baselines.
    \item \textbf{Section \ref{sec:related_cycle_consistency}} clarifies the fundamental distinctions between our temporal cycle-consistency framework and traditional spatial cycle-consistency models.
    \item \textbf{Section \ref{sec:applicability}} concludes with a rigorous discussion on the theoretical boundaries and broader applicability of the Cycle-World framework to other sequential modalities.
\end{itemize}

\section{Detailed Proofs of Theoretical Bounds}
\label{sec:detailed_proofs}

This section provides the complete mathematical derivations for the theoretical claims established in the main manuscript. We rely on the definitions of generative drift $e_n = \left\| \hat{z}_n - z_n \right\|$ and the single-step cycle-consistency distance $d_{cycle}^{(n)} = \left\| \hat{z}_{n-1} - R_\phi(\hat{z}_n) \right\|$.

\subsection{Proof of Theorem 1 (Cycle-Bounded Drift)}
\begin{proof}
Applying the reverse-Lipschitz condition (Assumption 2) to the generated state $\hat{z}_n$ and the ground-truth state $z_n$, we have:
\begin{equation*}
    e_n = \left\| \hat{z}_n - z_n \right\| \le C \left\| R_\phi(\hat{z}_n) - R_\phi(z_n) \right\|.
\end{equation*}
By applying the triangle inequality and adding/subtracting both the generated history $\hat{z}_{n-1}$ and the ground-truth history $z_{n-1}$ inside the norm on the right side, we obtain:
\begin{equation*}
    e_n \le C \left( \left\| R_\phi(\hat{z}_n) - \hat{z}_{n-1} \right\| + \left\| \hat{z}_{n-1} - z_{n-1} \right\| + \left\| z_{n-1} - R_\phi(z_n) \right\| \right).
\end{equation*}
Substituting the definitions of the cycle-consistency distance $d_{cycle}^{(n)}$, the previous step's generative drift $e_{n-1}$, and the inherent reverse approximation error $\epsilon_R$ (Assumption 1) yields the step-wise recurrence:
\begin{equation*}
    e_n \le C \left( d_{cycle}^{(n)} + e_{n-1} + \epsilon_R \right).
\end{equation*}
\end{proof}

\subsection{Proof of Corollary 1 (Long-Horizon Error Bound)}
\begin{proof}
Starting from Theorem 1, we recursively substitute the recurrence for $e_{n-1}$:
\begin{align*}
    e_n &\le C e_{n-1} + C \left( d_{cycle}^{(n)} + \epsilon_R \right) \\
        &\le C \left[ C e_{n-2} + C \left( d_{cycle}^{(n-1)} + \epsilon_R \right) \right] + C \left( d_{cycle}^{(n)} + \epsilon_R \right) \\
        &= C^2 e_{n-2} + C^2 \left( d_{cycle}^{(n-1)} + \epsilon_R \right) + C \left( d_{cycle}^{(n)} + \epsilon_R \right).
\end{align*}
Continuing this expansion down to the initial condition $e_0$ yields the summation: $e_n \le C^n e_0 + \sum_{i=1}^n C^{n-i+1} \left( d_{cycle}^{(i)} + \epsilon_R \right)$. Factoring out the maximum local loss $\delta_{cycle}$ resolves it into the closed-form geometric expression in Eq.2 of the main text.
\end{proof}

\subsection{Proof of Proposition 1 (Theoretical Advantage over Unconstrained Baselines)}
\begin{proof}
Applying Corollary 1 to both generation pipelines, their respective drift limits are:
\begin{align*}
    E_n^{unc} &= C^n e_0 + (\delta_{unc} + \epsilon_R) C \frac{C^n - 1}{C - 1}, \\
    E_n^{ours} &= C^n e_0 + (\delta_{cycle} + \epsilon_R) C \frac{C^n - 1}{C - 1}.
\end{align*}
Subtracting $E_n^{ours}$ from $E_n^{unc}$ immediately yields the reduction gap:
\begin{equation*}
    \Delta E_n = (\delta_{unc} - \delta_{cycle}) C \frac{C^n - 1}{C - 1}.
\end{equation*}
Since $\delta_{unc} > \delta_{cycle}$ and $C \ne 1$, it is evident that $\Delta E_n > 0$.

\end{proof}

\section{Algorithm for Cycle-Guided Inference}
\label{sec:algorithm}

As detailed in the main manuscript, Cycle-Guided Inference (CGI) serves as a zero-shot, training-free latent optimization strategy to enforce temporal reversibility in frozen foundational models. Algorithm \ref{alg:inference} provides the complete step-by-step pseudocode for this procedure. 

Specifically, during a designated optimization window $[T_{\text{start}}, T_{\text{end}}]$ within the autoregressive diffusion loop, CGI computes the Euclidean cycle discrepancy $\mathcal{D}$ using the frozen reverse corrector $R_\phi$. Because the reverse reconstruction is fully differentiable within the latent space, the gradient of $\mathcal{D}$ is iteratively backpropagated to actively refine the intermediate noisy latent $z_{n,t}$. This runtime rectification acts as a structural bottleneck, pruning non-physical artifacts before the final latent is permanently committed to the historical context buffer.

\begin{algorithm}[t]
\caption{Cycle-Guided Inference with Runtime Corrector}
\label{alg:inference}
\begin{algorithmic}[1]
\Require Forward generator $G_\theta$, Reverse corrector $R_\phi$, Diffusion transition function $\Psi$, Initial context $\hat{z}_0$, Sequence length $N$, Timesteps $\{t_T, \dots, t_1\}$, Optimization window $[T_{\text{start}}, T_{\text{end}}]$, Context noise $t_{\text{ctx}}$, Refinement iterations $K$, Step size $\eta$.
\Ensure Generated video latents $\hat{Z} = \{\hat{z}_1, \dots, \hat{z}_N\}$
\State Initialize history buffer $\mathcal{H} = \{\hat{z}_0\}$
\State Initialize output sequence $\hat{Z} = \emptyset$
\For{$n = 1$ to $N$}
    \State Sample initial noise $z_{n,t_T} \sim \mathcal{N}(0, \mathbf{I})$
    \For{$j = T$ down to $1$}
        \State Let $t = t_j$
        \If{$n > 1$ and $t \in [T_{\text{start}}, T_{\text{end}}]$} \Comment{Cycle Guidance Optimization Window}
            \State $z_{n,t}^{(0)} \gets z_{n,t}$
            \State Sample target noise $\epsilon \sim \mathcal{N}(0, \mathbf{I})$
            \For{$k = 0$ to $K-1$}
                \State $\hat{z}_{n|t} \gets G_\theta(z_{n,t}^{(k)}, t, \mathcal{H})$ \Comment{Predict clean latent}
                \State $R_\phi(\mathcal{F}(\hat{z}_{n|t}), t_{\text{ctx}}, \mathcal{H}_{\text{rev}})$ \Comment{Construct reverse context}
                \State $\tilde{z}_{n-1} \gets \mathcal{F}\left(R_\phi(\epsilon, t_T, \mathcal{H}_{\text{rev}})\right)$ \Comment{Predict predecessor from noise}
                \State $\mathcal{D} \gets \left\| \hat{z}_{n-1} - \tilde{z}_{n-1} \right\|_2^2$ \Comment{Compute cycle discrepancy}
                \State $z_{n,t}^{(k+1)} \gets z_{n,t}^{(k)} - \eta \nabla_{z_{n,t}^{(k)}} \mathcal{D}$ \Comment{Gradient descent update}
            \EndFor
            \State $z_{n,t} \gets \text{Detach}(z_{n,t}^{(K)})$ \Comment{Final refined state}
        \EndIf
        
        \State $\hat{z}_{n|t} \gets G_\theta(z_{n,t}, t, \mathcal{H})$ \Comment{Forward prediction for transition}
        
        \If{$j = 1$} \Comment{Final Denoising Step}
            \State $\hat{z}_n \gets \hat{z}_{n|t}$ \Comment{Final generated clean latent}
            \State $\mathcal{H} \gets \mathcal{H} \cup \{\hat{z}_n\}$ \Comment{Update KV cache}
            \State $\hat{Z} \gets \hat{Z} \cup \{\hat{z}_n\}$ \Comment{Update output sequence}
        \Else \Comment{Intermediate Denoising Step}
            \State Sample $\epsilon_{\text{trans}} \sim \mathcal{N}(0, \mathbf{I})$
            \State $z_{n,t_{j-1}} \gets \Psi(\hat{z}_{n|t}, \epsilon_{\text{trans}}, t_{j-1})$ \Comment{Standard diffusion transition}
        \EndIf
    \EndFor
\EndFor
\State \Return $\hat{Z}$
\end{algorithmic}
\end{algorithm}

\section{Implementation Details}
\label{sec:implementation_details}

In this section, we provide comprehensive implementation details of our proposed framework, covering dataset preparation, model configuration, training strategies, and hyperparameter settings.

\noindent\textbf{Data Preparation.} 
For the Ordinary Differential Equation (ODE) initialization phase in the diffusion forcing\cite{Diffusion-forcing-chen2024diffusion} process, the ground-truth ODE latents are acquired using the Wan2.1 14B model\cite{wan-2025wan}, as provided by CausVid~\cite{causvid-yin2025slow}. For the subsequent DMD training phase, we utilize a filtered and augmented version of the VidProM\cite{vidprom-wang2024vidprom,selfforcing-huang2025self} prompt dataset.

\noindent\textbf{Reverse Model Construction.} 
The reverse prediction model is initialized from the pre-trained Wan2.1 1.3B model. During the ODE initialization phase, we temporally flip the latent trajectories of the forward videos to initialize the reverse model. Subsequently, the model is trained utilizing the self-forcing paradigm. A key efficiency of our implementation is that, during training, we simply reverse the streaming latent outputs along the temporal dimension. This elegant operation allows us to directly reuse the bidirectional Wan model—which originally serves the forward prediction model—as both the teacher model and the critic, thereby eliminating the prohibitive computational cost of training a dedicated reverse teacher model.

\noindent\textbf{Forward Model Tuning.} 
The forward generator is built upon the pre-trained LongLive\cite{longlive-yang2025longlive} framework. We adapt streaming long tuning\cite{longlive-yang2025longlive} on 60-second video sequences, where each sequence contains a single prompt switch to encourage dynamic context transition. The forward model is fine-tuned using Low-Rank Adaptation (LoRA) on the LongLive weights. The training process spans a total of 3,000 iterations with a batch size of 4.

\noindent\textbf{Hyperparameter Configurations.} 
During the joint training phase, the weight for the cycle-consistency loss ($\lambda$) is empirically set to $0.1$. During the inference phase, the cycle guidance is applied across all denoising timesteps. For the runtime corrector, we set the number of optimization iterations per timestep to $K=1$, with a optimization step size of $\eta=10$. All other unmentioned architectural and optimization hyperparameters strictly adhere to the default configurations of LongLive.

\section{Physical Consistency Evaluation Metrics}
\label{sec:eval_metrics_physical}

To rigorously assess the physical realism of the generated videos, we utilize two complementary metrics: Physical Commonsense (PC) and Physical Alignment and Consistency Evaluation (PACE). We randomly sample 150 prompts from the VideoPhy\cite{videophy-bansal2024videophy,videophy2-bansal2025videophy} prompt set to synthesize the test videos for our evaluation protocol. For short-horizon evaluation on 5-second videos, the metrics are computed directly on the entire clip. For long-horizon evaluation on 60-second videos, we divide each generated sequence into non-overlapping 5-second chunks and randomly sample three distinct chunks per video for assessment.

\noindent\textbf{Physical Commonsense (PC).} 
The PC metric employs the automatic scoring model introduced in the VideoPhy benchmark, which assigns a discrete physical consistency score ranging from 1 to 5 to each video clip. To establish a stringent baseline for physical realism, we define the final PC score as the percentage of evaluated clips that achieve a score of 4 or higher.

\noindent\textbf{Physical Alignment and Consistency Evaluation (PACE).} 
To obtain a granular, human-aligned assessment of complex physical dynamics, we introduce PACE, a Multimodal-LLM-as-a-Judge metric. We utilize the Gemini model to evaluate the videos based on a dedicated prompt. The model is instructed to act as an expert assessor and output a comprehensive score from 0 to 100. The evaluation explicitly penalizes physical hallucinations by scrutinizing basic prompt compliance alongside four core physical dimensions: the realistic representation of gravity and mass, the natural dynamics of collisions, motion continuity (avoiding sudden teleportation or unnatural warping), and long-term temporal coherence. To facilitate automated parsing, the model is constrained to return the evaluation result strictly in JSON format. The exact prompt template provided to the model is as follows:

\begin{quote}
Act as an expert in video quality assessment and physics. Evaluate the provided video based on these criteria: (1) Prompt Compliance: Does the video content strictly follow the intended action and description? (2) Physical Consistency: Does the video adhere to real-world physical laws? Look for gravity and weight (do objects fall or move with realistic mass?), collisions (do interactions between objects look natural?), motion continuity (is there any sudden teleportation or unnatural warping?), and temporal coherence (does the scene remain consistent over time?). Provide a score from 0 to 100 and a concise justification for your rating. Return the result strictly in JSON format with 'score' and 'reason' keys.
\end{quote}

\section{Extended Evaluation on Physical Consistency}
\label{sec:extended_evaluation}

To comprehensively validate the effectiveness of our proposed framework in extremely long sequences, we first present the extended quantitative comparisons and ablation studies focusing on the 60-second video generation tasks.

\begin{table}[htbp]
\centering
\caption{Quantitative comparison of physical consistency on 60-second video generation.}
\label{tab:phys_comp_60s}
\begin{tabular}{lccc}
\toprule
\textbf{Method} & \textbf{PC} & \textbf{PACE} & \textbf{Average} \\
\midrule
Self Forcing & 65.67 & 49.96 & 57.82 \\
LongLive & 59.45 & \textbf{75.71} & 67.58 \\
Rolling Forcing & 61.69 & 74.59 & 68.14 \\
Infinity-RoPE & 63.18 & 72.77 & 67.98 \\
\textbf{Ours} & \textbf{70.90} & 75.42 & \textbf{73.16} \\
\bottomrule
\end{tabular}
\end{table}

\noindent\textbf{Superiority in Long-Horizon Physical Consistency.}
To evaluate the robustness of our method over extended sequences, we compare its physical consistency against state-of-the-art baselines on 60-second generation, as detailed in Table~\ref{tab:phys_comp_60s}. The results demonstrate that standard autoregressive approaches like Self Forcing struggle with complex physical dynamics over time, yielding an exceptionally low PACE score of 49.96 despite maintaining a moderate physical commonsense score. While long-context models such as LongLive and Rolling Forcing improve PACE, their overall physical consistency remains bounded around an average of 68. Our method effectively harmonizes prompt adherence with real-world physical laws, achieving the highest average score of 73.16. This superiority confirms that enforcing temporal reversibility acts as a critical mechanism to prevent physical hallucinations in extremely long videos.

\begin{table}[htbp]
\centering
\caption{Ablation study of physical consistency on 60-second video generation.}
\label{tab:ablation_60s}
\begin{tabular}{lccc}
\toprule
\textbf{Method} & \textbf{PC} & \textbf{PACE} & \textbf{Average} \\
\midrule
Baseline & 59.45 & 75.71 & 67.58 \\
+ CCL & 70.40 & 74.94 & 72.67 \\
+ CGI & 62.77 & \textbf{76.67} & 69.72 \\
\textbf{Ours (Cycle-World)} & \textbf{70.90} & 75.42 & \textbf{73.16} \\
\bottomrule
\end{tabular}
\end{table}

\noindent\textbf{Synergistic Effects of Cycle Constraints.} The 60-second ablation study reveals the distinct and complementary roles of the proposed modules. The baseline model exhibits a noticeable imbalance over extended horizons, maintaining a relatively high PACE score of 75.71 but struggling with foundational physical commonsense, as evidenced by a low PC score of 59.45. Incorporating the training-time cycle-consistent learning immediately provides a robust parametric foundation, which drastically raises the PC score to 70.40, albeit with a slight reduction in PACE to 74.94. This indicates that the parametric prior effectively enforces rigid structural laws but may slightly constrain unbridled dynamic variance. Conversely, applying the cycle guidance inference alone primarily preserves the high PACE metric at 76.67 but only marginally improves the PC score to 62.77, demonstrating that runtime optimization without a learned physical prior is insufficient to fundamentally correct structural physical violations. The full model integrates both mechanisms to achieve a powerful dual-phase synergy, maximizing the PC score to 70.90 and recovering the PACE score to 75.42, which culminates in the highest overall average performance of 73.16.
\section{More Qualitative Comparisons}

\subsection{Interactive Long Video Generation}
\label{sec:interactive_generation}

\begin{figure}[t]
    \centering
    \includegraphics[width=\textwidth]{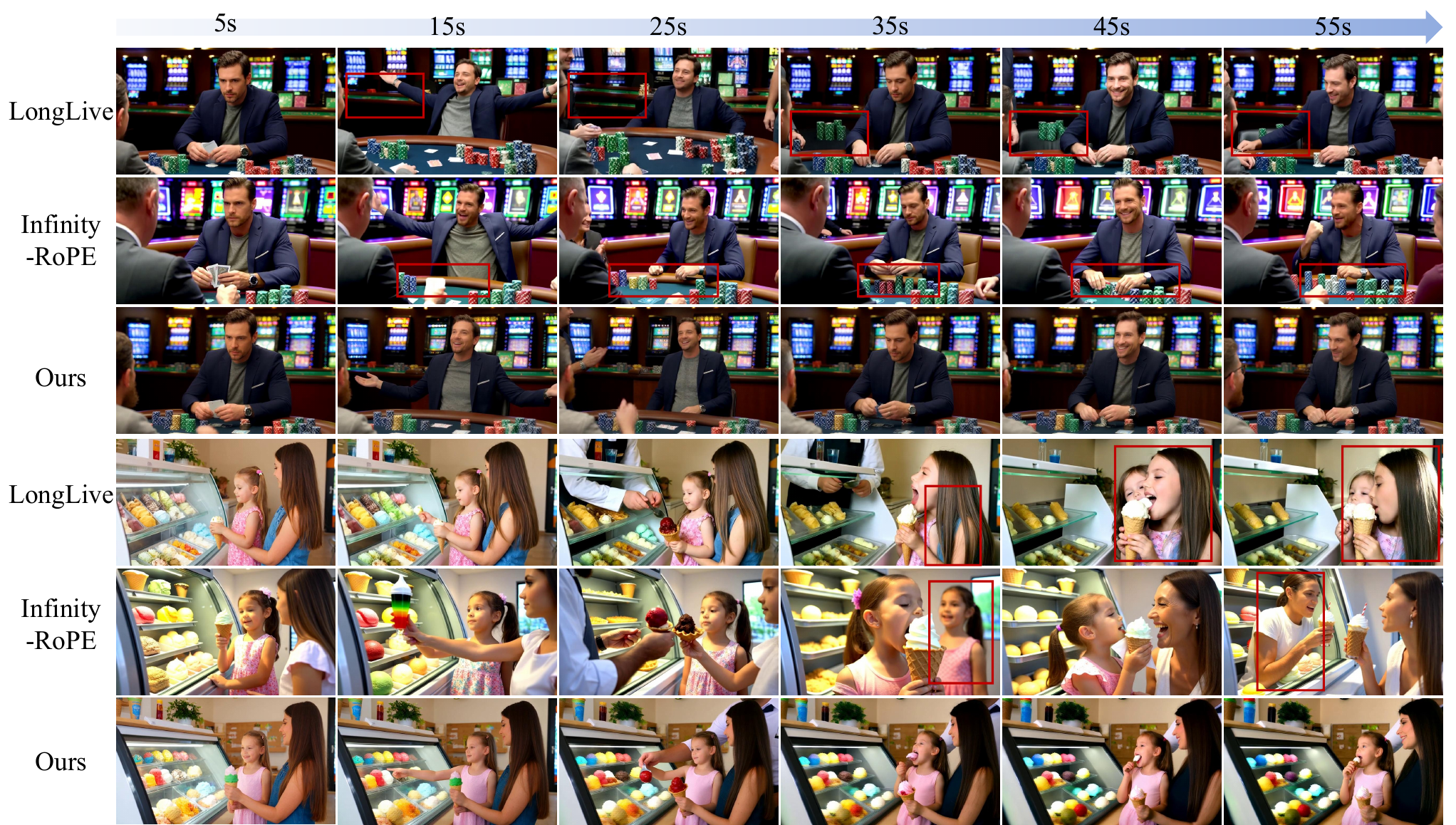}
\caption{Qualitative comparison of interactive long video generation over a 60-second horizon. The textual prompt is dynamically updated every 10 seconds. Cycle-World seamlessly interpolates new instructions while maintaining strict identity and background consistency, whereas baselines suffer from severe physical hallucinations and abrupt transitions.}
    \label{fig:interactive_generation}
\end{figure}

Beyond continuous prediction, a true video world model must support interactive, open-ended generation, allowing users to dynamically alter the future trajectory based on a shared historical context. However, dynamically changing text conditions during autoregressive generation often exacerbates structural collapse in standard models, as the sudden semantic shift disrupts the already fragile temporal continuity.

Leveraging the robust physical grounding provided by our Cycle-World framework, we evaluate its performance in interactive generation scenarios over a 60-second horizon, where the textual prompt is updated every 10 seconds. In this highly challenging setting, standard baselines such as Infinity-RoPE and LongLive struggle significantly. When confronted with semantic shifts or large-scale subject movements, these models frequently exhibit abrupt, unnatural scene transitions. Furthermore, they suffer from severe identity degradation and physical hallucinations, such as the sudden appearance or vanishing of background objects and characters. 

As illustrated in Fig. \ref{fig:interactive_generation}, Cycle-World exhibits exceptional adaptability and structural resilience. The runtime cycle critic explicitly penalizes physically impossible transitions, ensuring that the background remains stable and the subject's identity is strictly preserved despite the semantic branch. Unlike forward-only baselines that hallucinate entirely new entities when the prompt shifts, our model seamlessly interpolates the fluid dynamics required by the new instruction while maintaining strict temporal reversibility. This demonstrates that enforcing cycle consistency establishes a highly robust latent manifold, unlocking stable and interactive control for open-ended world simulation.

\subsection{Single-Prompt Long Video Generation}
\label{sec:single_prompt_generation}

\begin{figure}[t]
    \centering
    \includegraphics[width=\textwidth]{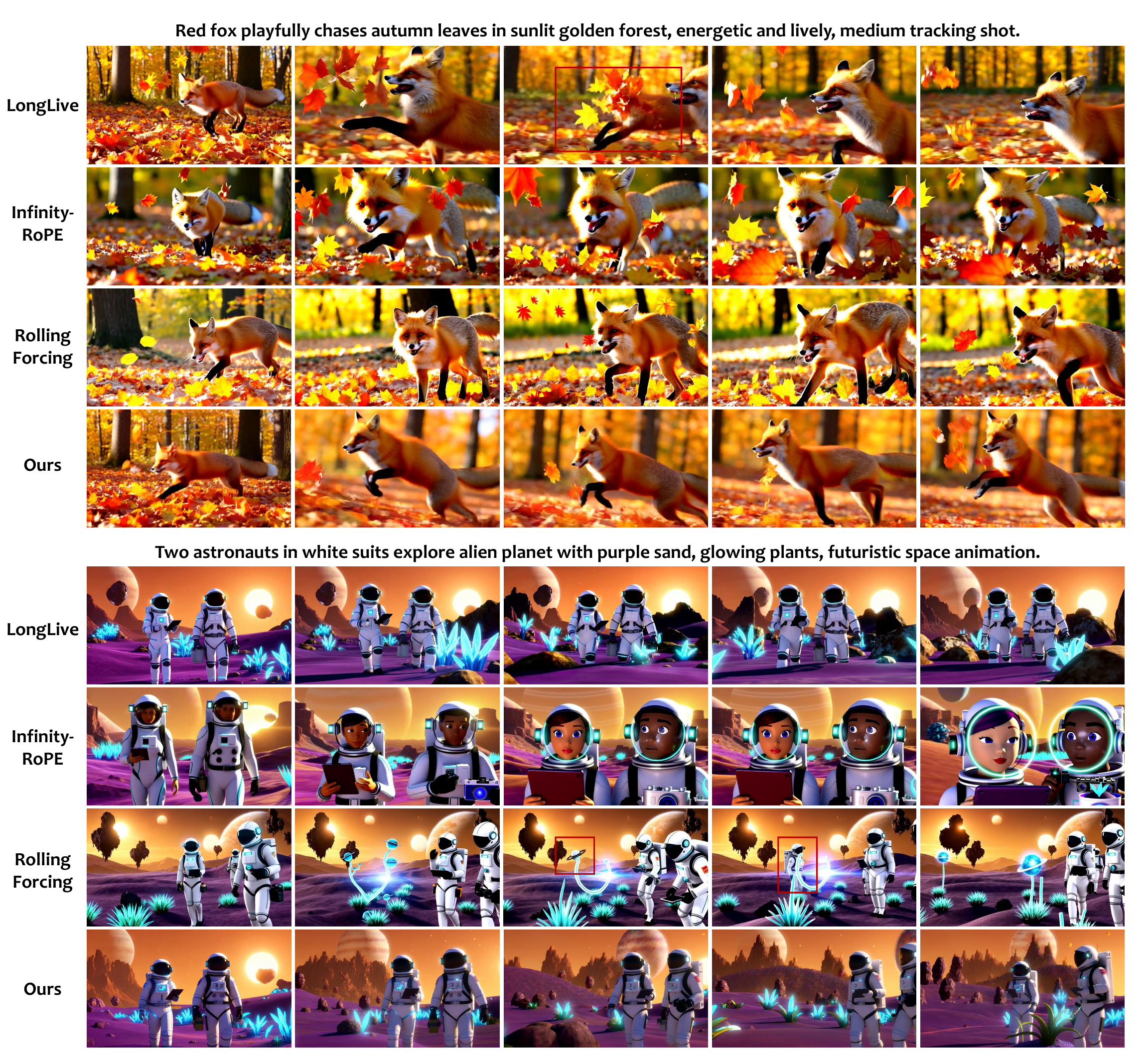}
\caption{Qualitative comparison of 60-second single-prompt video generation. Compared to LongLive, Infinity-RoPE, and Rolling Forcing, our method effectively prevents object interpenetration and hallucinatory artifacts, achieving an optimal balance between long-term visual consistency and rich motion dynamics.}
    \label{fig:qualitative_comparison_single_prompt}
\end{figure}

To further demonstrate the robustness of our framework, we compare Cycle-World against state-of-the-art long video autoregressive models, namely LongLive, Infinity-RoPE, and Rolling Forcing, under a 60-second single-prompt generation setting. Generating extended sequences without intermediate text guidance exposes the critical vulnerabilities of existing methods over time.

As shown in Fig. \ref{fig:qualitative_comparison_single_prompt}, LongLive suffers from physical hallucinations, including the spontaneous manifestation of non-existent subjects and unnatural object interpenetration (structural clipping). Infinity-RoPE, while attempting to maintain structural coherence, exhibits significantly degraded motion dynamics and fails to preserve subject consistency over extended periods. Similarly, Rolling Forcing exhibits a limited dynamic range and tends to hallucinate abrupt, out-of-context entities.

In stark contrast, our Cycle-World framework successfully achieves an optimal balance between visual consistency and rich motion dynamics. By rigorously enforcing temporal reversibility, our model prevents unconstrained generative drift. It ensures that the primary subject, background integrity, and natural physical interactions are preserved throughout the entire 60-second duration, without sacrificing the amplitude and realism of the generated motion.

\section{More Qualitative Results}
\label{sec:more_qualitative}

\begin{figure}[t]
    \centering
    \includegraphics[width=\textwidth]{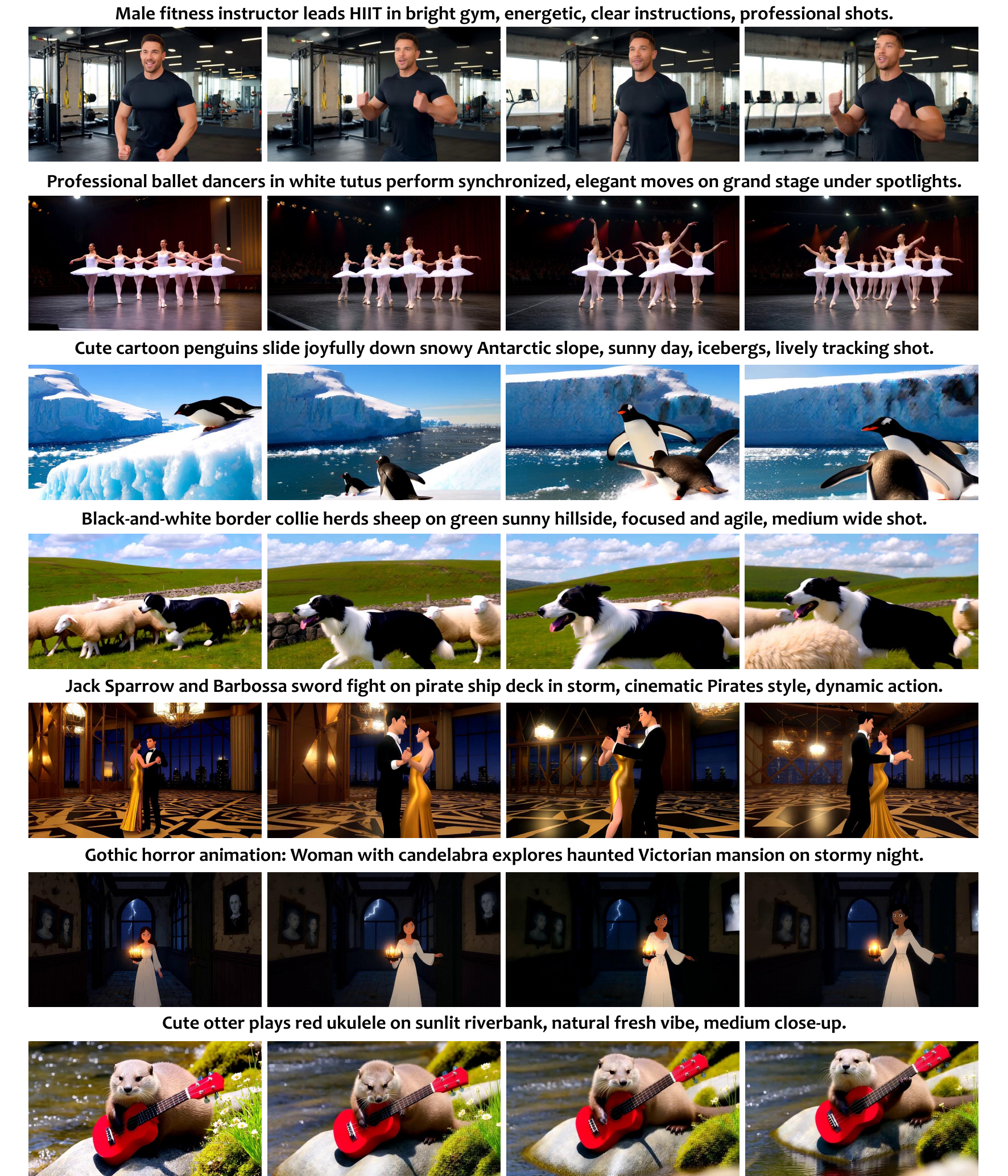}
\caption{Additional qualitative results of Cycle-World across diverse scenes and complex motions, demonstrating its high visual fidelity and robust spatiotemporal consistency.}
    \label{fig:more_qualitative_results}
\end{figure}

We provide additional qualitative results in Fig. \ref{fig:more_qualitative_results} to showcase the versatility and high visual fidelity of Cycle-World across diverse scenes and complex physical motions. These examples further corroborate the efficacy of our cycle-consistency framework. By inherently bounding the autoregressive prediction errors, Cycle-World is capable of producing highly stable, structurally sound, and aesthetically pleasing long-horizon video simulations across a wide variety of open-domain prompts.

\section{Hyperparameter Analysis}
\label{sec:hyperparameters}

We analyze the key hyperparameters of our framework to understand their impact on the trade-off between visual quality and physical consistency. Specifically, we examine the training cycle loss weight alongside the temporal distribution and step size of the inference cycle guidance.

\begin{table}[htbp]
\centering
\caption{Impact of the cycle loss weight during training (evaluated on 5-second generation).}
\label{tab:lambda_tuning}
\begin{tabular}{lccccc}
\toprule
Weight ($\lambda$) & Total & Quality & Semantic & PC & PACE \\
\midrule
0.2 & 82.89 & 84.80 & 75.27 & 68.66 & 75.80 \\
0.1 & 84.05 & 85.83 & 76.93 & 64.18 & 79.04 \\
0.05 & 84.47 & 86.56 & 76.13 & 56.72 & 76.53 \\
\bottomrule
\end{tabular}
\end{table}

Table~\ref{tab:lambda_tuning} shows the effect of the cycle loss weight during training on 5-second video generation. The results reveal a trade-off between aesthetic quality and structural adherence. A lower weight of 0.05 improves visual and semantic scores but causes a sharp decline in physical commonsense to 56.72. This indicates that the model struggles to stay on the physical manifold. A higher weight of 0.2 enforces strict cycle consistency and raises the PC score to 68.66. However, this regularization penalizes visual fidelity, reducing the total VBench score to 82.89. A weight of 0.1 balances these aspects, yielding the highest PACE score of 79.04 while preserving competitive video quality.

\begin{table}[htbp]
\centering
\caption{Effect of the temporal distribution of optimization steps during inference.}
\label{tab:opt_distribution}
\begin{tabular}{lccccc}
\toprule
Optimization Strategy & Total & Quality & Semantic & PC & PACE \\
\midrule
1 step across all 4 timesteps & 84.36 & 86.14 & 77.25 & 69.40 & 81.91 \\
2 steps on the first 2 timesteps & 84.22 & 85.89 & 77.52 & 69.40 & 81.77 \\
4 steps on the 1st timestep & 84.07 & 85.90 & 76.78 & 68.66 & 76.79 \\
\bottomrule
\end{tabular}
\end{table}

Table~\ref{tab:opt_distribution} details the effect of distributing a fixed budget of four optimization steps across the denoising timesteps during inference. Concentrating all four steps at the initial timestep results in the weakest performance, particularly lowering the PACE score to 76.79. Distributing the steps evenly by applying one optimization per timestep yields the best results across both visual and physical metrics. Continuous gradient guidance along the generation trajectory is more effective for maintaining temporal coherence than isolated early intervention.

\begin{table}[htbp]
\centering
\caption{Sensitivity analysis of the inference gradient guidance step size.}
\label{tab:step_size}
\begin{tabular}{lccccc}
\toprule
Step Size ($\eta$) & Total & Quality & Semantic & PC & PACE \\
\midrule
5 & 84.20 & 85.86 & 77.58 & 68.66 & 76.79 \\
10 & 84.36 & 86.14 & 77.25 & 69.40 & 81.91 \\
15 & 83.92 & 85.62 & 77.12 & 70.15 & 82.27 \\
\bottomrule
\end{tabular}
\end{table}

The gradient guidance step size controls the strength of the runtime corrector. As shown in Table~\ref{tab:step_size}, increasing the step size from 5 to 15 steadily improves the model's adherence to physical constraints, with PC and PACE reaching 70.15 and 82.27. An aggressive step size of 15 compromises general generation quality, leading to lower total and semantic VBench scores. A step size of 10 provides an optimal configuration that maintains physical alignment without sacrificing the visual and semantic integrity of the sequence.
\section{Plug-and-Play Extensibility of Cycle-Guided Inference}
\label{sec:plug_and_play}

A practical benefit of Cycle-Guided Inference (CGI) is its extensibility. Since the runtime corrector uses gradient-based latent refinement without altering the architecture of the forward generator, it integrates directly into existing autoregressive video models. This inference strategy applies to any forward-generation framework that operates within the same VAE latent space as the pre-trained reverse model.

To evaluate this adaptability, we apply the frozen reverse-prediction corrector to two standard forward-only baselines, CausVid and Self-Forcing. As Figure \ref{fig:plug-and-play} shows, adding CGI consistently improves both models in the 60-second long-horizon setting. 

Qualitative results highlight several improvements from this integration. CGI enhances the global spatiotemporal consistency of the generated sequences, preserving subject identities and background details over extended frames. It also corrects physical anomalies common in unconstrained autoregressive models. For example, in a dynamic chasing scene, unmodified baselines often generate a background that incorrectly moves forward relative to the running subjects. By enforcing temporal reversibility, the CGI module corrects this motion error, ensuring the background recedes naturally. 

Furthermore, while standard forward-only models suffer from generative drift over long contexts, applying CGI delays structural collapse. By removing accumulated artifacts at each autoregressive step, this strategy extends the effective generation length of the underlying baselines. This results in longer, structurally stable video sequences without requiring additional model retraining.

\begin{figure}[t]
    \centering
    \includegraphics[width=\textwidth]{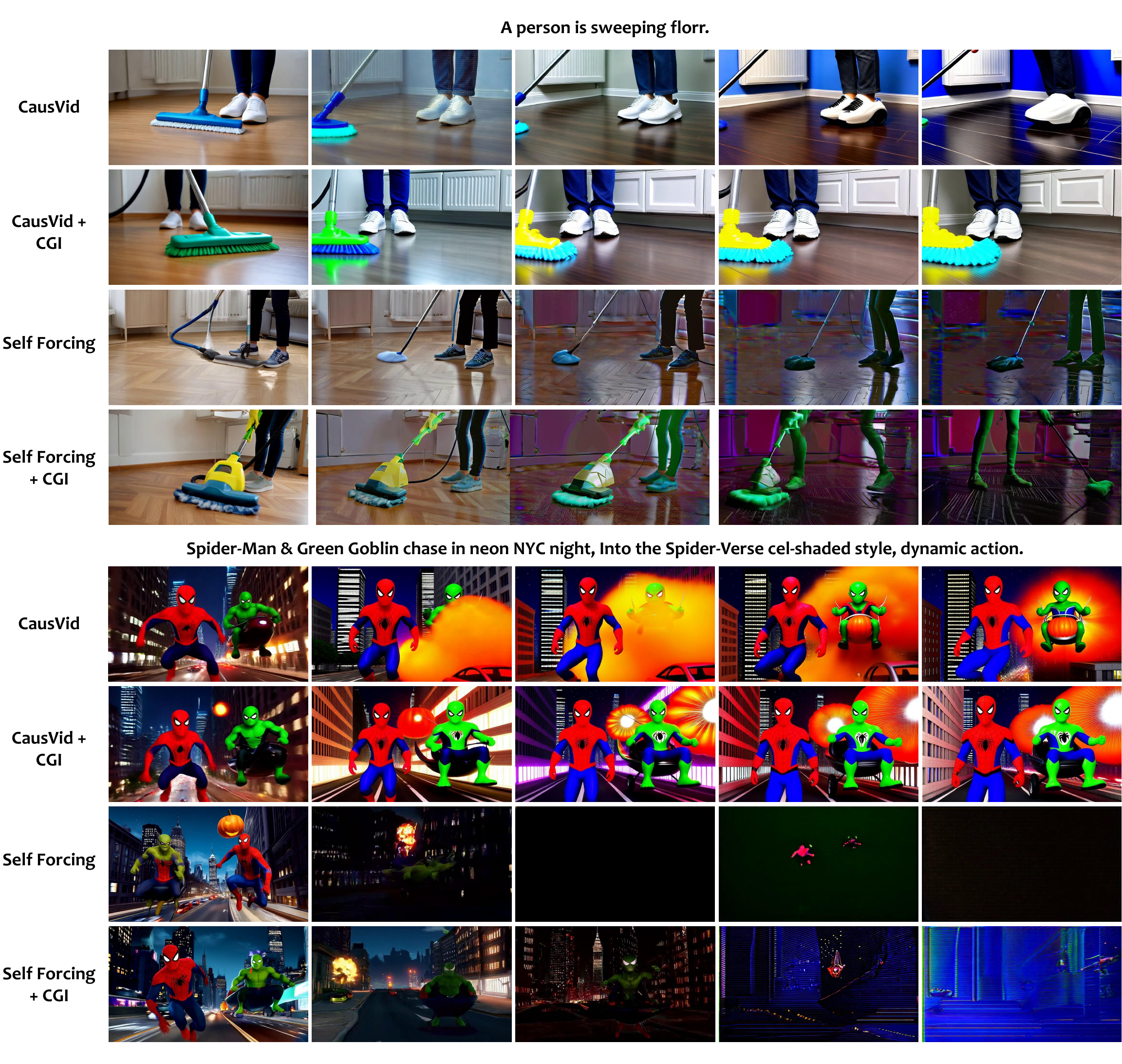}
\caption{Qualitative results of integrating Cycle-Guided Inference (CGI) into CausVid and Self-Forcing. The integration improves spatiotemporal consistency and corrects physical motion anomalies without model retraining.}
    \label{fig:plug-and-play}
\end{figure}

\section{Distinctions from Traditional Cycle Consistency}
\label{sec:related_cycle_consistency}

The concept of cycle consistency has been widely explored in computer vision, most notably in unpaired image-to-image translation frameworks such as CycleGAN\cite{zhu2017unpaired}. These traditional methods introduce a cycle-consistency objective to learn bijective mappings between two distinct spatial or stylistic domains. By ensuring that an image translated to a target domain can be accurately reconstructed back to its original domain, these models effectively bypass the requirement for strictly paired training data. The constraint primarily operates spatially, focusing on preserving texture, geometry, and structural content across different artistic or sensor modalities.

Our Cycle-World framework fundamentally diverges from these traditional applications in both its core objective and operational domain. Rather than mapping between different visual domains, our method enforces cycle consistency strictly along the temporal axis within a single continuous domain. The primary goal is not stylistic translation, but rather bounding the compounding generative drift inherent in long-horizon autoregressive synthesis. We conceptualize cycle consistency as a fundamental physical constraint, grounded in the observation that valid natural dynamics and causal events must be temporally reversible.

Functionally, this distinction translates to divergent architectural implementations. While traditional frameworks employ two symmetric cross-domain spatial generators, Cycle-World pairs a forward autoregressive generator with a temporal reverse-prediction model. This configuration constructs a step-wise cycle across sequential states. By actively minimizing the discrepancy between a historical state and its reverse-predicted reconstruction from a future state, our approach serves as an intrinsic physical regularizer. This temporal cycle explicitly penalizes structural hallucinations and non-physical motion artifacts, thereby maintaining long-term causal coherence rather than mere spatial fidelity.

\section{Applicability to General Temporal Modalities}
\label{sec:applicability}

While this study grounds the Cycle-World framework in long-horizon video synthesis, its theoretical formulation is intrinsically modality-agnostic. The core mechanism of bounding autoregressive drift via cycle consistency relies entirely on the topological properties of the latent manifold rather than the specific visual nature of the data. Consequently, this framework can be extended to other sequential domains, provided the underlying data distribution strictly adheres to the principle of local temporal reversibility.

Audio generation represents a highly compatible domain for this extension. Acoustic signals, whether speech, music, or environmental sounds, are continuous physical waveforms governed by mechanical laws and temporal causality. Current autoregressive audio models frequently experience compounding errors that manifest as rhythmic degradation, phase shifts, or the gradual loss of speaker identity over extended contexts. Because acoustic dynamics preserve short-term historical information within their local temporal window, training a reverse acoustic predictor is mathematically well-posed. Applying cycle-guided inference to audio latent spaces could actively correct these deviations, ensuring the generated sequence remains anchored to a natural acoustic manifold without requiring architectural changes to the base audio model.

Beyond perceptual modalities, the cycle-consistency paradigm holds significant potential for general temporal forecasting tasks in physically grounded environments. Predictive models in autonomous driving, robotic kinematics, and molecular dynamics simulate spatial-temporal states that are primarily governed by classical mechanics. In these Newtonian systems, state transitions are inherently time-reversible. A valid future state must contain sufficient deterministic information to deduce its immediate predecessor. By utilizing cycle guidance as a runtime physical regularizer, predictive models would be compelled to respect these mechanical constraints, thereby preventing the unconstrained divergence of simulated trajectories over long time horizons.

However, a rigorous theoretical boundary limits the universal application of this framework to all autoregressive tasks. The efficacy of Cycle-World is strictly contingent upon the reverse-predictability assumption. This mechanism cannot be generalized to highly entropic, discrete, or lossy sequential processes where the arrow of time introduces severe information collapse. In domains such as abstract text generation or financial market forecasting, state transitions often represent many-to-one mappings where multiple distinct past contexts can converge into an identical current state. Under such macroscopic irreversible conditions, the backward mapping becomes fundamentally ill-posed, and the reverse approximation error would violate the theoretical bounds required for our theorem to hold. Therefore, the applicability of Cycle-World is rigorously confined to continuous, physically grounded, or information-preserving latent manifolds where temporal inversion remains locally deterministic.
\end{document}